\documentclass{article}
\usepackage[preprint]{neurips_2024}

\usepackage[utf8]{inputenc} 
\usepackage[T1]{fontenc}    
\usepackage{hyperref}       
\usepackage{url}            
\usepackage{booktabs}       
\usepackage{amsfonts}       
\usepackage{nicefrac}       
\usepackage{microtype}      
\usepackage{xcolor}         

\usepackage{natbib}

\usepackage{bbm}
\usepackage{natbib}
\setcitestyle{numbers,square}
\usepackage{tikz}
\usetikzlibrary{shapes,arrows}
\usepackage{amsmath}
\usepackage{amsthm,amsfonts}
\usepackage{graphicx}
\usepackage{cleveref}
\usepackage{autobreak}

\crefname{section}{Section}{secs}
\crefname{theorem}{Theorem}{thms}
\crefname{assumption}{Assumption}{asss}
\crefname{lemma}{Lemma}{lems}
\usepackage{geometry}
\usepackage{booktabs}
\usepackage{algorithm}
\usepackage{algpseudocode}
\usepackage{algpseudocode}
\usepackage{float} 
\usepackage{subfig} 
\usepackage{url}
\usepackage{listings}
\usepackage{setspace}
\usepackage{cancel}
\usepackage{mathrsfs,color}
\usepackage[normalem]{ulem}
\graphicspath{ {Figures/} }
\newtheorem{definition}{Definition}
\newtheorem{theorem}{Theorem}
\newtheorem{lemma}[theorem]{Lemma}    

\newtheorem{corollary}[theorem]{Corollary}
\theoremstyle{definition}
\newtheorem{remark}{Remark}
\newtheorem{assumption}{Assumption}

\newcommand{\CC}{\mathcal{C}}

\newcommand{\R}{\mathbb{R}}

\newcommand{\N}{\mathbb{N}}

\newcommand{\wenyu}[1]{\textcolor{blue}{}}
\newcommand{\mmp}[1]{\textcolor{red}{}}
\newcommand{\MM}{\mathcal{M}}
\newcommand{\XX}{\mathcal{X}}
\newcommand{\tilXX}{\tilde{\mathcal{X}}}
\newcommand{\tiln}{\tilde{n}}
\newcommand{\phihat}{\hat{\varphi}}
\newcommand{\That}{\hat{T}}

\newcommand{\rrr}{\mathbb{R}}

\newcommand{\wfs}{{\rm wfs}}

\makeatletter

\newcommand{\Rmnum}[1]{\expandafter\@slowromancap\romannumeral #1@}
\makeatother  

\title{How well behaved is finite dimensional Diffusion Maps embedding?}

\author{%
  Wenyu Bo\thanks{--} \\
  Department of Statistics\\
  University of Washington\\
  Seattle, WA 98195  \\
  \texttt{wbo@uw.edu} \\
  \And
  Marina Meil\u{a}\\
  Department of Statistics\\
  University of Washington\\
  Seattle, WA 98195  \\
  \texttt{mmp@stat.washington.edu}
}

\begin{document}   

\maketitle

\begin{abstract}
	Under a set of assumptions on a family of submanifolds $\subset \R^D$, we derive a series of geometric properties that remain valid after finite-dimensional diffusion maps (DM), including almost uniform density, finite polynomial approximation and reach. Leveraging these properties, we establish rigorous bounds on the embedding errors introduced by the DM algorithm is $O\left((\frac{\log n}{n})^{\frac{1}{8d+16}}\right)$. Furthermore, we quantify the error between the estimated tangent spaces and the true tangent spaces over the submanifolds after the DM embedding,
	\begin{align*}
		\sup_{P\in \mathcal{P}} \mathbb{E}_{P^{\otimes \tiln}} \max_{1\leq j \leq \tiln} \angle (T_{Y_{\varphi(M),j}}\varphi(M),\hat{T}_j)\leq C \left(\frac{\log n }{n}\right)^\frac{k-1}{(8d+16)k},
	\end{align*}
	which providing a precise characterization of the geometric accuracy of the embeddings. These results offer a solid theoretical foundation for understanding the performance and reliability of DM in practical applications.
	
\end{abstract}

\section{Introduction}
\label{sec:intro}
The Diffusion Maps (DM) embedding\cite{COIFMAN20065,belkin2003laplacian}, a dimensionality reduction technique that captures the geometric structure of data by constructing a diffusion process among data points, is central to manifold learning from
samples. \cite{berard1994embedding} showed that in the limit of
infinite dimension it is an isometric embedding, while
\cite{portegies2016embeddings} showed that almost isometry can be achieved
with a finite number $m$ of eigenfunctions, where this $m$ depends on
manifold geometric properties. \cite{bates2014embedding} showed that if
isometry is not required, then the sufficient embedding dimension by
Laplacian eigenfunctions depends on dimension, injectivity radius, Ricci curvature and volume, thus that it can be
arbitrarily larger than the Whitney embedding dimension of $2d$.

The DM embedding is widely used for non-linear dimension reduction as
the Diffusion Maps algorithm \cite{COIFMAN20065,lafon2006diffusion,coifman2008diffusion,arvizu2016dimensionality}, which embeds a sample into $m$
dimensions by the eigenvectors of $L_n$, a $n\times n$ matrix
estimator of the Laplace-Beltrami operator $\Delta$.
For instance, DM is frequently used to analyze high-dimensional single-cell RNA sequencing data, revealing cell differentiation trajectories and underlying biological patterns\cite{haghverdi2015diffusion}, and it can order cells along their differentiation paths, enabling accurate reconstruction of branching developmental processes\cite{haghverdi2016diffusion}. 
In chemistry, DM can extract the dynamical modes of high-dimensional simulation trajectories, furnishing a kinetically low-dimensional framework\cite{ferguson2011nonlinear}, it can identify the collective coordinates of rare events in molecular transitions\cite{boninsegna2015investigating}.
In Astronomy, DM is applied to estimate galaxy redshifts from photometric data, demonstrating comparable accuracy to existing methods\cite{freeman2009photometric}.

Stimulated in part by the practical applications of DM, the consistency of  finite sample estimators of $\Delta$, and of their eigenfunctions has been intensely studied. Many early works discussed the pointwise consistency for empirical operator with smooth functions on $M$ and they derived convergence (with different rates) \cite{von2008consistency,belkin2008towards,gine2006empirical,10.5555/3104322.3104459,hein2007graph}, which provide illuminating results for subsequent studies. Not only $\Delta$ itself, the consistency of its embedding are also widely studied, they discussed how the eigenvectors and eigenvalues of the empirical Laplacian with specific graph (K-NN, $\epsilon$-graph) converge to the
eigenfunctions and eigenvalues of $\Delta$ in the different norm ($L^2$, $L^\infty$)\cite{cheng2022eigen,garcia2020error,wang2015spectral,belkin2006convergence}, also there are some other special convergence been studied. For instance, convergence in Lipschitz norm\cite{calder2022lipschitz}, embedding using heat kernel\cite{jones2008manifold,DUNSON2021282}

This paper completes the picture/advances the understanding of the DM
embedding in the infinite and finite sample case, by considering
geometric properties of a manifold, such as smoothness, injectivity
radius, reach, volume, diameter, and examining to what extent, and
under the conditions that diffusion time $t$ is small enough and embedding dimension $m$ is large enough, the DM preserves
these properties. In other words, if the original manifold $\MM$ is
``well behaved'', what can be said about its DM embedding $\varphi(\MM)$?

More specifically, we set a series of geometric assumptions on a family of compact manifolds (\cref{sec:problem}), then we study the smoothness of $\varphi$ by examining its Sobolev norms
(\cref{sec:smooth}) and use it to implies the existence of local parameterization, then we derived consistent properties that hold uniformly for all manifolds in the family after DM embedding, like sampling densities and reach (\cref{phiM}). In \cref{sec:convergence}, we identified the correspondence between the embedding via eigenfunctions and the embedding via eigenvectors of the graph matrix and address the convergence between them, then estimate the error between the eigenfunctions and eigenvectors to quantify the noise in the positions of the embedded points. Finally, we use the error we estimated and the geometric quantities transformed by the DM to approximate the tangent space and obtained the convergence rate.(\cref{sec:tangent})

\section{Background, challenges and assumptions}
\label{sec:problem}

\subsection{Manifolds and the Diffusion Map}
For basic definitions reader should consult \cite{do1992riemannian,lee2003introduction,lee2018introduction}. Here we consider sub-manifolds $M\subset \rrr^D$ which we will generically call {\em manifolds}. Note that the ambient dimension $D$ will not appear throughout the paper, only Euclidean distances $|x-y|$ with $x,y\in \rrr^D$ will appear, hence $D$ can actually be infinite, and $\rrr^D$ could become a Hilbert space.

We also assume that $M$ is a closed manifold with smoothness of class $C^k$, as described in \cref{as:smooth}. This assumption ensures the existence of a local polynomial expansion up to order $k$ in the neighborhood of each point on $M$. Such expansions satisfy specific regularity constraints, enabling precise characterization of the local geometry and supporting rigorous analytical derivations. These properties are fundamental for establishing the theoretical results and ensuring consistency in subsequent computations.

In this work, we focus on several key geometric quantities, including the tangent space, reach, Riemannian metric, and geodesic curves. These quantities play critical roles in understanding the geometry of the manifold and in estimating and reconstructing manifolds. While these concepts are broadly defined in the general framework in Riemannian geometry, considering our sampling points are in the submanifolds in $\R^D$, we restrict our attention to their specialized formulations for submanifolds embedded in Euclidean space. This setting simplifies their definitions and aligns with the analytical and computational techniques employed in this study. For example, the tangent space is characterized as the hyperplane tangent to $M$ at a given point, and the Riemannian metric is induced from the ambient Euclidean metric. 

\begin{definition}[Tangent Space]
	\upshape
	Let $M$ be a smooth manifold and $p$ be a point of $M$. If a linear map $v:C^\infty(M)\rightarrow \R$ satisfies
	\begin{align}
		v(fg)=f(p)vg+g(p)vf \quad \text{for all } f,g \in C^\infty(M),
	\end{align}
	then we call $v$ a derivation at $p$. The set of all derivations at $p$ is called the tangent space to $M$ at $p$, denoted by $T_pM$. 
\end{definition}

For a manifold $M\subset \R^D$, let $\iota: M \hookrightarrow \R^D$ be the inclusion map. Let $(U,\textbf{x})$ be the chart containing $p$, and $\hat{U}:=\textbf{x}(U) \subset \R^d$. We have $\hat{\iota}:=\iota \circ \textbf{x}^{-1}:\hat{U} \overset{\textbf{x}^{-1}}{\rightarrow} U \overset{\iota}{\hookrightarrow} \R^D$ is a local representation near $p$ by
\begin{align}
\left(\textbf{x}_1(p),\dots,\textbf{x}_d(p)\right)\rightarrow\left(p_1,\dots,p_D\right).
\end{align}
Then $\left\{\frac{\partial\hat{\iota}}{\partial x_1}(\textbf{x}(p)),\dots,\frac{\partial\hat{\iota}}{\partial x_d}(\textbf{x}(p))\right\}$ span a $d$-dimensional subspace at $\iota(p)$ of $\R^D$, which is the common definition of tangent space at $p$ in the Euclidean case. 

\begin{definition}[Geodesic Normal Coordinate]
	Let $p \in U \subset M$ and $\left. \exp_p\right|_V : V\subset T_pM \rightarrow U$ be a diffeomorphism, and there is a basis isomorphism $B$ between $\R^d$ and orthonormal basis $\left\{b_i\right\}$ for $T_p M$ by $B(x_1,\dots,x_d)=\sum_{i=1}^d x_i b_i$. Then $U$ is the normal neighborhood with normal coordinate $\left(\left. \exp_p\right|_V \circ B\right)^{-1}$.
\end{definition}

Since $d_0 (\exp_p)$ is the identity map, thus by inverse function theorem, such neighborhood $V$ always exists, then it is well defined.

\begin{definition}[Riemannian Metric]
	\upshape
	A Riemannian metric $g$ on a manifold $M$ assigns to every $p$ a inner product $g_p(\cdot,\cdot)$ on $T_pM$ which is smooth in the following sense: For a chart $(U,\textbf{x})$ containing $p$, with $\textbf{x}^{-1}(x_1,\dots,x_d)= q \in U$ and $\frac{\partial}{\partial x_i}(q)=d\textbf{x}^{-1}_q(0,\dots,1,\dots,0)$, then $\left\langle\frac{\partial}{\partial x_i}(q),\frac{\partial}{\partial x_j}(q) \right\rangle = g_{ij}(x_1,\dots,x_d)$ is smooth on $\hat{U}$. 
\end{definition}

In our setting, manifolds are embedded into $\R^D$ via $\iota$, then Riemannian metric automatically inherits Euclidean metric, i.e. $g_p(v,w)=\left\langle d\iota_p v, d \iota_p w  \right\rangle $.

Riemannian metric defines a measure on $M$, called the {\em Riemannian measure} via the {\em volume form}. If $M$ is compact, it has a finite volume $Vol(M)$. We can now have other measures on $M$, which are absolutely continuous w.r.t. the Riemannian measure, through their densities $f:M\rightarrow (0,\infty)$.  

In statistics/manifold learning, it is assumed that we are given an i.i.d. sample $X_n=\{x_1,\ldots x_n\}\subset M$ from $f$. 

\begin{definition}[Reach\cite{federer1959curvature}]
	\upshape
	The reach of a subset $A$ of $\R^n$ is the largest $\tau$ (possibly $\infty$) such that if $x \in\R^n$ and the  distance from $x$ to $A$ is smaller that $\tau$, then $A$ contains a unique point nearest to $x$. For more detail and properties, refer \cite{aamari2019estimating}.
\end{definition}
The reach of a subset $A \subset \R^n$ 
provides a measure of its local geometric regularity. Specifically, if the reach of 
$A$ is $\tau>0$, then every point within a distance less than $\tau$ from $A$ has a unique nearest point on $A$. Geometrically, this implies that 
$A$ does not exhibit sharp corners, cusps, or regions of high curvature within the specified radius. In particular: For convex subsets of $\R^n$, the reach is infinite, reflecting the absence of curvature bounds or sharp features. For smooth submanifolds, the reach is inversely related to its maximum principal curvature. Intuitively, the reach corresponds to the radius of the smallest osculating ball that fits locally around.

\begin{definition}[Laplacian-Beltrami Operator]
	\upshape~
	
The Laplacian-Beltrami operator is the linear operator $\Delta:C^\infty(M)\rightarrow C^\infty(M)$ defined by
\begin{align}
	\Delta f = div(grad f).
\end{align}
\end{definition}

It is a classical result that the eigenvalues of the Laplace-Beltrami operator $-\Delta$ on a Riemannian manifold $M$ form a non-decreasing spectrum, i.e. \begin{align}
	0\leq \lambda_0 \leq \lambda_1\leq \lambda_2 \dots,
\end{align}
where each eigenvalue is repeated according to its multiplicity. When $M$ is compact,  the spectrum of $-\Delta$ is discrete, and each eigenvalue has finite multiplicities. If $M$ is compact and without boundary, the smallest eigenvalue $\lambda_0 =0 $ , and its eigenspace consists of constant functions. Furthermore, the regularity theory for elliptic operators ensures that all eigenfunctions are smooth, i.e., they belong to $C^\infty(M)$.

\paragraph{The Diffusion map}
\label{embed}
Let $M$ be a smooth, closed manifold of class $\mathcal{C}_3$ embedded in the
(possibly high-dimensional) Euclidean space $\rrr^D$ and $\Delta_{\mathcal{M}}$ be Laplacian
operator on $M$ with eigenvalues $0=\lambda_0 \leq \lambda_1 \leq
\cdots$. We consider eigenfunctions corresponding these eigenvalues:
$e_0,e_1,\cdots$, and we normalize them such that
$\|e_i\|_{L^2}=1$. It is easy to check $e_0 \equiv const$. Then we have the following embedding theorem.

\begin{theorem}
	\label{embeddingthm}
Let $\MM$ be the set of $d$ dimensional, closed Riemannian manifolds whose Ricci curvature is bounded from below by $\kappa$, injectivity radius is bounded from below by $\iota$, and the volume is bounded from above by $V$, we define
$\varphi: M\subset \mathbb{R}^D \rightarrow N\subset \mathbb{R}^m$ as the following:
\begin{equation}
	\varphi(x)=(2t)^\frac{d+2}{4}\sqrt{2}(4\pi)^\frac{d}{4}(e^{-\lambda_1 t}e_1(x),e^{-\lambda_2 t}e_2(x),\cdots,e^{-\lambda_m t}e_m(x)) \in \mathbb{R}^m,
\end{equation}
then there exists a $t_0=t_0(d,\kappa,\iota,\epsilon)$ such that for all $0<t<t_0$, there exists a $N_0(d,\kappa,\iota,\epsilon,V,t)$ such that if $N>N_0$, then for all $M\in \mathcal{M}(d,\kappa,\iota,V)$, the map above is an embedding of $M$ into $\mathbb{R}^N$, and $1-\epsilon < \|d \varphi_p v\| < 1+\epsilon$ where $\|v\|=1$.\cite{portegies2016embeddings}
\end{theorem}

\begin{lemma}
	\label{kernelerror}
We use the same notation as in the preceding theorem. Then for $\epsilon^\prime>0$, there exists $N_1=N_1(d,\kappa,\iota,V,\epsilon^\prime,t_0)$ such that when $N\geq N_1$ and $t_0 \leq t^\prime\leq 4$, we have 
\begin{align}
	\|K_N(t^\prime,p,\cdot)-K(t^\prime,p,\cdot)\| \leq \epsilon^\prime,
\end{align}
where $K(t,p,q)$ is the heat kernel:
\begin{align}
	K(t,p,q)=\sum_{i=0}^{\infty} e^{-\lambda_i t} e_i(p) e_i(q),
\end{align}
and $K_N(t,p,q)$ is the truncated heat kernel:
\begin{align}
	K_m(t,p,q)=\sum_{i=0}^{m} e^{-\lambda_i t} e_i(p) e_i(q).
\end{align} 
\end{lemma}

We select $t=t_0/2$, $t^\prime = t_0$ (if $t_0>4$, we set $t=2$ and $t^\prime=4$) and $m=\max\left\{N_0,N_1\right\}+1$ to make the results above both hold.

In addition, $\varphi:M\rightarrow \varphi(M)$ is homeomorphism, so $\dim M=\dim \varphi(M)$, and $\varphi$ is an embedding, also an immersion, thus $\varphi$ is local diffeomorphism, and $\varphi$ is bijective, so $\varphi$ is diffeomorphism. Thus $\varphi(M)$ is a smooth manifold.\cite{lee2003introduction}

\mmp{
\begin{enumerate}
	\item We have $n$ samples $\left\{x_{1:n}\right\} \overset{i.i.d.}{\sim} M$, where $M$ is a d-dimensional manifold, and $x_i \in \mathbb{R}^D$.
	\item Using Embedding Algorithm (DM), we have $\left\{y_{1:n}\right\}=\hat{\varphi}(x_{1:n}) \in \mathbb{R}^m$ where $m\geq d$.
	\item We estimate Tangent Space $\Tilde{T_i}$, and want to estimate $\Vert Err(\Tilde{T_i}) \Vert$.
\end{enumerate}
shall we be consistent about $x\in M, \,y\in\varphi(M)$?
}

Given a finite sample $\mathcal{X}_n=\left\{x_1,\dots,x_n\right\}$ the DM algorithm 
 constructs a similarity matrix to measure the pairwise relationships between data points,where the commonly used kernel is the Gaussian kernel, defined as $k(x_i, x_j) = \exp(-\|x_i - x_j\|^2 / h)$, where $h$ is a scale parameter, and this similarity matrix is then normalized, we denote it as Laplacian graph which is the approximation of Laplacian-Beltrami Operator. By performing eigenvalue decomposition on normalized Laplacian graph, the algorithm extracts the dominant eigenvalues and eigenvectors. These eigenvectors, scaled by their corresponding eigenvalues, define the diffusion coordinates, providing a low-dimensional embedding of the data. This embedding preserves the global geometry of the dataset while emphasizing its intrinsic structure.

In this paper, we will focus mainly on the geometric properties of the
DM $\varphi(M)$, w.r.t. the original manifold $M$. The results we obtain
will be useful in characterizing the output of the DM algorithm in
finite sample settings, and we apply them specifically to the
estimation of the tangent subspace $\left\{T_{\varphi(x_i)}\varphi(x_i)\right\}_{i=1}^n$.

\cite{berard1994embedding} have shown that, in the limit of large $m$ and small $t$, $\varphi(M)$ is isometric with $M$, and that this is possible approximatively with a finite $m$ for manifolds with bounded diameter and Ricci curvature bounded from below. In spite of these seemingly encouraging results, the DM can be highly unstable even for apparently ``nice'' manifolds. The intuitive explanation is the fact that, even if $M$ is compact and smooth to order $k$, the local interactions between the manifold curvature and the manifold reach (note that these are not independent quantities) can exert a strong influence of the Laplacian eigenfunctions. 

\subsection{Assumptions}

Thus, in predcting smoothness (w.r.t. Sobolev norms) and geometric properties of $\varphi(M)$, one needs to consider $\tau_M, \iota_M, \kappa_M$ in addition the smoothess of $M$ (made more precise below). We will give the assumptions for the set of manifolds we will perform DM, and then discuss basic properties of $M$ based on these assumptions. We will see all these assumptions intuitively ensure our manifolds have very good shape, which can help us avoid extremely bizarre situations. And we will use them to derive existence of the geometric bound after DM, which are critical in estimating the convergence rate.

\begin{assumption}[Curvature]
	\label{as:kappa}
	The absolute value of sectional curvature of $M$ is bounded by $\kappa$, that is $|K(u,v)|\leq \kappa$. which immediately implies Ricci curvature of $M$ is bounded below by $-\kappa(d-1)$ since $Ric(v,v)=\frac{1}{d-1}\sum_{i=1}^{d}K(v,x_i)$ where $\left\{v, x_2,\dots,x_d\right\}$ are orthonormal basis. This assumption will be used in \cref{embed}, \cref{sec:phim-lambdam}, \cref{thm:upperbound}, \cref{thm:lowerbound}. 
  \end{assumption} 
Bounded Ricci curvature prevents the submanifold from having extreme geometric variations or "infinite negative curvature" in any direction. This geometric control ensures that local neighborhoods behave predictably, which is crucial for DM that rely on local structure.

\begin{assumption}[Reach]
	\label{as:reach}
  The reach $\tau_M$ of $M$ is bounded below by $\tau_{min}$. This assumption will be used in \cref{sec:reach}.
\end{assumption}
Positive minimum reach is crucial because it ensures the manifold doesn't come too close to self-intersecting and has bounded curvature, making it possible to reliably reconstruct the manifold from discrete samples.

\begin{assumption}[Volume]
	\label{ass:V}
  The volume of $M$ is bounded below by $V_1$ and bounded above by $V_2$. This assumption will be used in \cref{embed}, \cref{sec:phim-lambdam}.
\end{assumption}

\begin{assumption}[Smoothness]
	\label{as:smooth}
For $k$ and $\textbf{L}:=\left(L_\perp,L_3,\dots,L_k\right)$, we assume there exists a local one-to-one parameterization $\Psi_p$ for all $p \in M$:
	\begin{align}
		\Psi_p :B_{T_pM}(0,r) \rightarrow M \quad \text{by} \quad \Psi_p(v)=p+v+N_p(v)
	\end{align}
	for some $r\geq \frac{1}{4L_\perp}$ with $N_p(v) \in C^k(B_{T_pM}(0,r),\mathbb{R}^D)$ such that 
	\begin{align}
		N_p(0)=p,\  d_0N_p=0,\ \|d_v^2N_p\| \leq L_\perp, \ \|d_v^iN_p\|\leq L_i\ \text{for }i=3,\dots,k
	\end{align}
	holds for all $\|v\|\leq \frac{1}{4L_\perp}$. This assumption will be used in \cref{sec:phim-L}
\end{assumption}

Smoothness of order $k$ implies manifolds can be approximated locally by multilinear map over tangent space with bounded norm.

In addition, we make the following more technical regularity
assumption that will be used in \cref{sec:phim-higher-deri},\cref{sec:reach}.

\begin{assumption}[Regular Condition]
	We assume that our estimating manifold family is a subset of $\MM$ such that the uniform constants $C_1(\MM)$, $C_2(\MM)$ ensure that \cref{star} holds for $M$.
\end{assumption}

\begin{assumption}[Christoffel Symbols]
	\label{as:Christoffel}
	In the normal coordinate chart, the derivatives with order not above $k-2$ of Christoffel symbol (including Christoffel symbol itself) have the uniform upper bound only depending on its order, i.e.
	\begin{align}
		|\frac{\partial^l \Gamma_{ij}^k}{\partial x_{i_1}\cdots\partial x_{i_l}}| \leq C(l) \quad \text{for } 1\leq i,j,k\leq d\ \text{ and }\  l\leq k-2
	\end{align}
\end{assumption}
Christoffel symbol measures the change of a vector along a curve due to curvature, thus bounded Christoffel symbol control the curvature in some manner.

Let $\mathcal{M}(d,\kappa,\tau_{min},V,k,\textbf{L},\Gamma)$ be the set of compact connected submanifolds $M \subset \mathbb{R}^D$, with dimension $d$ satisfying \cref{as:kappa}--\cref{as:Christoffel}. For simplicity, in the rest of our paper we always assume that the manifold is $d$-dimensional and the ambient dimension is $D$, we also use the abbreviation $\mathcal{M}$ for $\mathcal{M}(d,\kappa,\tau_{min},V,k,\textbf{L},\Gamma)$, and sometimes we include some of the parameters above to indicate that the assumptions corresponding to these parameters are satisfied.

We also assume that the sampling density on $M$ does not deviate too much from uniform.
\begin{assumption}[Density]
\label{ass:density}
  Let $\mathcal{P}_{f_{\min},f_{\max}}$ denote the set of distribution $P$ with support on $M \in \mathcal{M}$, and the density function $f$ of $P$ with respect to Hausdorff measure such that $0<f_{\min}\leq f \leq f_{\max}<\infty$. This assumption will be used in \cref{sec:density}.
\end{assumption}

All the assumptions above, perhaps with the exception of \cref{as:smooth} and \cref{as:Christoffel}, are generically present in the manifold learning literature.

\subsection{Direct Consequences of Assumptions}
Here we list some direct results for $M \in \MM$ from assumptions which will also be used in our following proof.

\begin{corollary}[Complete Manifold]
	Any compact Riemannian manifold is geodesically complete according to Hopf-Rinow theorem.\cite{do1992riemannian}
\end{corollary}

\begin{corollary}[The injectivity radius]
	\label{injradius}
	The injectivity radius $\iota_M$ of $M$ is bounded below by $ \pi \tau_M$\cite{aamari2019estimating}, which implies $\iota_M$ is bounded by $\pi \tau_{\min}$ for all $M \in \MM$. We use it to make sure DM is an embedding.
\end{corollary}
The injectivity radius bounded from below implies we can find normal coordinate chart with the uniform radius.

\begin{corollary}[Diameter]
	For $M \in \mathcal{M}$, we have
	\begin{align}
		\operatorname{diam}(M)\leq \frac{C_d}{\tau_M^{d-1}f_{\min}}\leq \frac{C_d}{\tau_{\min}^{d-1}f_{\min}}
	\end{align}
	where $C_d$ is a constant only depending on $d$\cite{aamari2018stability}. Thus the diameter of $M \in \mathcal{M}$ have an uniform upper bound. We will use it in \cref{sec:phim-lambdam}.
\end{corollary}

\subsection{Flowchart}
Here we use a flowchart to show the algorithm we want to run. In order to run this algorithm, we need to ensure a series of manifold properties, which will be proved in the next chapter.

\tikzstyle{startstop} = [rectangle,rounded corners, minimum width=3cm,minimum height=1cm,text centered, draw=black,fill=red!30]
\tikzstyle{io} = [trapezium, trapezium left angle = 70,trapezium right angle=110,minimum width=3cm,minimum height=1cm,text centered,draw=black,fill=blue!30]
\tikzstyle{process} = [rectangle,minimum width=3cm,minimum height=1cm,text centered,text width =3cm,draw=black,fill=orange!30]
\tikzstyle{decision} = [diamond,minimum width=3cm,minimum height=1cm,text centered,draw=black,fill=green!30]
\tikzstyle{arrow} = [thick,->,>=stealth]
\begin{figure}[H]
	\centering
	\tikzstyle{startstop} = [rectangle,rounded corners, minimum width=3cm,minimum height=1cm,text centered, draw=black]
	\tikzstyle{io} = [trapezium, trapezium left angle = 70,trapezium right angle=110,minimum width=3cm,minimum height=1cm,text centered,draw=black]
	\tikzstyle{process} = [rectangle,minimum width=3cm,minimum height=1cm,text centered,text width =3cm,draw=black]
	\tikzstyle{decision} = [diamond,minimum width=3cm,minimum height=1cm,text centered,draw=black]
	\tikzstyle{arrow} = [thick,->,>=stealth]
	\begin{tikzpicture}[node distance=2cm]
		\node (geo-assump) [startstop, text width=6cm] {Manifold Family $\MM$ with the intrinsic dimension is $d$, reach $\geq$ $\tau$, Ricci curvature $\geq -\kappa(d-1)$, \par  volume  $V \in (V_1, V_2)$, the smoothness condition (\cref{as:smooth}), Christoffel symbols and their higher order derivatives are bounded.};
		
		\node (data-assump) [startstop,right of =geo-assump , xshift=5cm,  text width=6cm] {Data points $\mathcal{X}_n=\left\{X_{1},\dots,X_{n}\right\}$ with uniform density.};
		
		\node (data-compute) [process, below of=data-assump,yshift=-0.5cm,  text width=6cm ] {Compute bandwidth $h=(\frac{\log n}{n})^{\frac{1}{4d+13}}$.};
		
		\node (inject) [process,below of=geo-assump,yshift=-0.5cm, text width=6cm] {Control injectivity radius $\geq \iota$, and the diameter $\leq \text{diam}$, and select error $\epsilon$. };
		
		\node (tm) [process,below of=inject, text width=7cm] {Compute $t=\min\left\{t_0(d,\kappa,\iota,\epsilon),4,\frac{\iota^2}{4}\right\}$, $\epsilon^\prime=\frac{(4\pi t)^{-\frac{d}{2}}}{8}\exp\left(-\frac{\beta^2t}{4}-\frac{2\sqrt{3dt}\beta}{3}\right)$, $m=\max\left\{N_0(d,\kappa,\iota,\epsilon,V,\frac{t}{2}),N_1(d,\kappa,\iota,V,\epsilon^\prime,t)\right\}+1$
		};
		
		\node (matrix) [process,below of=data-compute,yshift=0cm] {Construct $W$ and $D$ and $L_n$.};
		
		\node (eigen) [process,below of=tm,yshift=-0.5cm, xshift=3.5cm,text width=6cm ] {Compute the first $m$ eigenvalues $\lambda_i$ and normalized eigenvectors $e_i$ of $L_n$.};
		
		\node (embed) [process,below of=eigen,yshift=0cm,text width=6cm  ] {Embed $X_j$ to $(t)^\frac{d+2}{4}\sqrt{2}(4\pi)^\frac{d}{4}\left(e^{-\lambda_i t/2}e_i(X_j)\right)_{i=1}^m$.};

		\node (stop) [startstop,below of=embed] {Estimate tangent space at $\tiln=n^{\frac{d}{(8d+16)k}}$ embedding points};
		
		\draw [arrow] (geo-assump) -- (inject);
		\draw [arrow] (data-assump) -- (data-compute);
		\draw [arrow] (inject) -- (tm);
		\draw [arrow] (data-compute) -- (matrix);
		\draw [arrow] (tm) -- (eigen);
		\draw [arrow] (matrix) -- (eigen);
		\draw [arrow] (eigen) -- (embed);
		\draw [arrow] (embed) -- (stop);
	\end{tikzpicture}
    \caption{Flowchart}  
\label{fig:flowchart} 
\end{figure}
\section{Properties of the diffusion maps embeeding $\varphi(\mathcal{M})$}
\label{phiM}

In this section, we will discuss the properties of $\varphi(M)$, and
we will first give some uniform estimate of geometric quantities for
all $M \in \mathcal{M}$, and we will prove that the family
$\varphi(\mathcal{M}) \in \mathcal{M}(\tau_{\varphi(M)},\textbf{L}^\prime)$
where $\textbf{L}^\prime:=\left(L_\perp^\prime,\dots,L_k^\prime\right)$ and the
density function of $\varphi(M)\in  \varphi(\mathcal{M})$ admits an
uniform upper and lower bounds.


\subsection{Bounds on $\lambda_m$ and the (higher) derivatives of $\phi$}
\label{sec:smooth}

Estimating and uniformly bounding the $k$-th derivative of $\phi$,
$\|d^k\varphi\|$, plays an important role in controlling other
geometric quantities such as $\tau_{min}$ and proving existence of
local ono-to-one parametrization of $\varphi(\mathcal{M})$. Recall that $k$ is the smoothness of $\MM$. Because the $i$-th component of map $\varphi:M \rightarrow \mathbb{R}^m$
is $e_i$ multiple of $e^{-\lambda_i t_0/2}$ which is a bounded scalar when $t_0$ is fixed, we can ignore
it when we estimate the upper bound.

We notice that $\varphi(M)$ is a $d$-dimensional manifold embedded in $\mathbb{R}^m$, then $d\varphi_p$ is a map such that 
\begin{align}
	d\varphi_p:T_p M \rightarrow T_{\varphi(p)}\varphi(M)\hookrightarrow\mathbb{R}^m
\end{align} 
and we can treat the tangent vector in $\mathbb{R}^m$ as a tangent vector in $T_\varphi(p) \varphi(M)$ through the natural isomorphism. This isomorphism maps local basis to local basis, thus this map is also isometric. In the following, we will consider $T_{\varphi(p)} \varphi(M) \cong \mathbb{R}^d$ and $T_{\varphi(p)} \varphi(M) \subset \mathbb{R}^m$ as equivalent through this mapping.
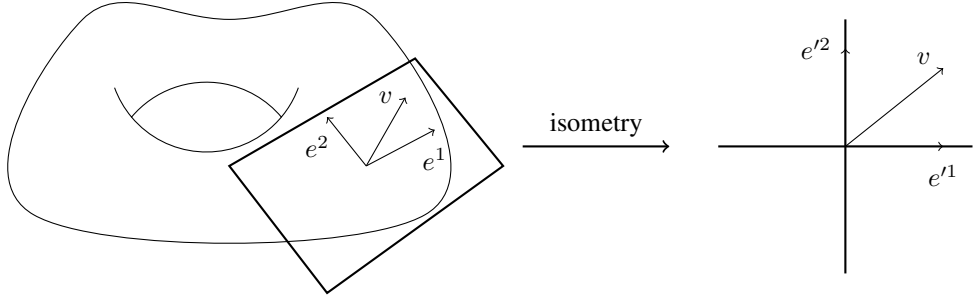
\begin{figure}[H]
	\centering
	\begin{tikzpicture}[scale=1.3]
		\draw[smooth cycle,tension=.7] plot coordinates{(0,0) (0.5,2) (2,2) (3.5,2) (4,0)};
		\coordinate (A) at (1,1);
		\draw (A) arc(140:40:1) (A) arc(-140:-20:1) (A) arc(-140:-160:1);

		\draw[thick] (2,0.5) -- (3,-0.8) -- (4.8,0.5) -- (3.9,1.6) -- cycle;
		\draw[->] (3.4,0.5) -- (4.1,0.87);
		\draw[->] (3.4,0.5) -- (3.0,1);
		\node at (4.1,0.57) {\(e^1\)};
		\node at (2.9,0.7) {\(e^2\)};
		
		\draw[->] (3.4,0.5) -- (3.8,1.2);
		\node at (3.6,1.2) {\(v\)};
		
		\draw[->, thick] (5,0.7) -- (6.5,0.7) node[midway, above] {isometry};
		
		\draw[thick] (7,0.7) -- (9.6,0.7);
		\draw[thick] (8.3,-0.6) -- (8.3,2);
		
		\draw[->] (8.3,0.7) -- (9.3,0.7);
		\draw[->] (8.3,0.7) -- (8.3,1.7);
		\node at (9.3,0.4) {\(e'^1\)};
		\node at (8.0,1.7) {\(e'^2\)};
		
		\draw[->] (8.3,0.7) -- (9.3,1.5);
		\node at (9.1,1.6) {\(v\)};
		
	\end{tikzpicture}
	\caption{Isometry}
\end{figure}

\subsubsection{Upper Bound of $\lambda_m$}
\label{sec:phim-lambdam}

The most classical result for estimating eigenvalues is Weyl's law:
\begin{align}
N(\lambda)\sim \frac{\omega(d) Vol(M) \lambda^\frac{d}{2}}{(2\pi)^{d}},
\end{align}
where $N(\lambda)$ is the number of eigenvalues less than or equal to $\lambda$. 

If we ask $\lambda=\lambda_k$, the above is equivalent to
\begin{align}
	\lambda_k^{d/2} \sim \frac{(2\pi)^d k}{\omega(n) Vol(M)}
\end{align}

Weyl's law provides an asymptotic expression for the eigenvalues of the Laplace-Beltrami operator, offering profound insights into their growth rates. However, in our paper, strict control for eigenvalues is essential. This requires not only asymptotic estimates but also rigorous upper bounds that hold universally. A notable result for the upper bound was established by Li-Yau, who derived explicit upper bounds for the eigenvalues under geometric constraints, which align with the assumptions in our study.

We estimate the upper bound of eigenvalues. If the lower bound of the Ricci curvature $\kappa_M$ of $M$ is less than $0$, we have 
\begin{align}
	\lambda_m \leq 
	\left\{
	\begin{array}{l}
		\displaystyle \frac{(2\beta+1)^2}{4}(-\kappa_M) + 4(1+2^\beta)^2\pi^2 \left(\frac{\sinh \sqrt{-\kappa_M}\text{diam}}{\sqrt{-\kappa_M}\text{diam}}\right)^\frac{2d-2}{d} \\
		\hspace{10em} \times \left((m+1)\frac{\omega(d-1)}{d}\frac{1}{V}\right)^\frac{2}{d} \\[1.5ex]
		\text{when } d=2(\beta+1), \beta = 0,1,2\dots \\[2.5ex]
		\displaystyle \frac{(2\beta+2)^2}{4}(-\kappa_M) + 4(1+\pi^2)(1+2^{2\beta})^2 \left(\frac{\sinh \sqrt{-\kappa_M}\text{diam}}{\sqrt{-\kappa_M}\text{diam}}\right)^\frac{2d-2}{d} \\
		\hspace{10em} \times \left((m+1)\frac{\omega(d-1)}{d}\frac{1}{V}\right)^\frac{2}{d} \\[1.5ex]
		\text{when } d=2\beta+3, \beta = 0,1,2\dots
	\end{array} 
	\right. ,
\end{align}

if $M$ has non-negative Ricci curvature $\kappa_M$, we have
\begin{align}
	\lambda_m \leq (d+4)d^{1-\frac{2}{d}}\left(\frac{m+1}{V}\omega(d-1)\right)^\frac{2}{d} ,
\end{align}
where $d$ is the dimension of $M$, $V$ is the volume of $M$ and $\omega(n)$ is the volume of $S^n$ in $\mathbb{R}^{n+1}$.\cite{esteigenvalue}

Since for $M \in \mathcal{M}$, the volume is greater that $V_1$ and the lower bound of the Ricci curvature is greater than $\kappa$, so $\frac{1}{V}\leq \frac{1}{V_1}$ and $-\kappa_M\leq -\kappa$, thus we have 
\begin{align}
	\lambda_m \leq \left\{
	\begin{array}{cc}
		C(m,d,V)\quad &\text{if } \kappa_M \text{ is non-negative} \\
		C(m,d,\kappa,\text{diam},V)\quad  &\text{if } \kappa_M \text{ is negative}
	\end{array}
	\right. .
\end{align}

For $M \in \mathcal{M}$, we can control $d$, $\kappa$, $\text{diam}$, $V$ uniformly, so we can find a constant $C$ only depending on $m$ for all $M \in \mathcal{M}$ such that
\begin{align}
	\lambda_m \leq C(m) .
\end{align}

\subsubsection{Estimate of Higher Order Derivatives}
\label{sec:phim-higher-deri}
In this subsection, we will estimate the $L^\infty$ norm of higher order derivatives for eigenfunctions to control the operator norm of $d N^\prime$ in the next subsection.  

To estimate the $L^\infty$ norm of derivatives for eigenfunctions, we first choose the coordinate chart to be normal coordinate chart to make our derivatives well defined. We take derivative of $e(x): M \rightarrow \R$ in the following sense: Let $x \in U$ be the normal coordinate chart $(U,\textbf{x})$ and $\hat{U}=\textbf{x}(U)$, then $\hat{e}(x)=e\circ \textbf{x}^{-1}:\hat{U}\rightarrow R$ be the coordinate representation of $e(x)$. We treat $\partial_i e(x)$ as $\partial_i \hat{e}(x)$, and for higher order derivatives as well.

Let $\gamma$ be a $2d$ multi-index with $|\gamma|:=\sum_i\gamma_i\leq k$, and $x,y \in U$ be a geodesic normal coordinate of $M$. We denote $\sum_{\lambda_j\leq \lambda} e_j(x) e_j(y)$ as $e(x,y,\lambda)$. Let $\gamma = \left(\gamma_1,\dots,\gamma_{2d}\right)$, then the $\gamma$ order derivative for $e_j(x) e_j(y)$ is 
\begin{align}
	\partial^\gamma_{x,y} e_j(x) e_j(y) = \frac{\partial^{|\gamma_1+\cdots+\gamma_{d}|}e_j(x)}{\partial x_1^{\gamma_1} \cdots \partial x_d^{\gamma_d}} \frac{\partial^{|\gamma_{d+1}+\cdots+\gamma_{2d}|}e_j(y)}{\partial y_1^{\gamma_{d+1}} \cdots \partial y_d^{\gamma_{2d}}}
\end{align}

We have\cite{2004Xuhigherorder}
\begin{align}
	|\partial^\gamma_{x,y}e(x,y,\lambda)|\leq C_\gamma (1+\lambda)^{\frac{d+|\gamma|}{2}}.
\end{align}

To have the estimate for derivatives of eigenfunctions, let $\alpha$ be a $d$ dimension multi-index, and using the inequality above with $\gamma = \left(\alpha , \alpha\right)$, then we have 
\begin{align}
	|\partial^\alpha e_\lambda (x)|^2\leq\sum\limits_{\lambda_j\leq \lambda} |\partial^\alpha e_j(x)|^2 \leq C_\alpha (1+\lambda)^{d/2+|\alpha|}
\end{align}
so
\begin{align}
	\label{deriinequ}
	\|e_\lambda\|_{H^\infty_k}=\max\limits_{|\alpha|\leq k}|\partial^\alpha e_\lambda(x)|\leq \sup\limits_{|\alpha|\leq k}\left\{C_\alpha\right\}(1+\lambda)^{k/2+d/4}
\end{align}
which implies that $\|e_j\|_{H_l^\infty}\leq C \lambda_j^\frac{2l+d}{4}$ for $e_j:M \rightarrow \mathbb{R}$. Therefore, $|\frac{\partial^l e_j}{\partial x_{i_1}\cdots\partial x_{i_l}}| \leq C \lambda_j^\frac{2l+d}{4}$ for $l\leq k$ with respect to normal coordinate.

To estimate the norm of $d\varphi$, we can derive its local coordinate representation and compute the covariant derivatives in the normal chart since  choosing different chart does not have the influence on the norm. Now we denote $\varphi$ as $j$-th component of $\varphi$ out of simplifying the notation and $(U,\textbf{x})$ is the normal chart.

For the first order covariant derivative, we have
\begin{align}
	\nabla \varphi (X) = \nabla_X \varphi = X\varphi
\end{align} 
thus 
\begin{align}
	\nabla \varphi = \varphi_i dx^i\quad  \text{with}\ \varphi_i = \partial_i \varphi
\end{align}
so for first order covariant derivative, the local coordinate representation is first order derivative of local representation of $\varphi$.

For second order covariant derivative, which is also known as covariant Hessian. We have 
\begin{align}
\nabla^2 \varphi (X,Y)= \nabla^2_{Y,X}=\nabla_Y(\nabla_X \varphi)-\nabla_{\nabla_Y X}\varphi=X(Y\varphi)-(\nabla_X Y)\varphi
\end{align}
thus in local coordinates,
\begin{align}
	\nabla^2 \varphi = \varphi_{ij} dx^i \otimes dx^j \quad \text{with}\ \varphi_{ij}=\partial^2_{ij}\varphi-\Gamma^k_{ij}\partial_k \varphi,
\end{align}
with $i,j,k\in\{1,2,\ldots m\}$. 
For the third order coordinate derivative, we have 
\begin{align}
	\nabla^3 \varphi (X,Y,Z)&=(\nabla_Z \nabla^2 \varphi) (X,Y)\\
	&=Z(\nabla^2 \varphi (X,Y))-\nabla^2\varphi (\nabla_Z X,Y)-\nabla^2 \varphi (X,\nabla_Z Y)\\
	&=Z(X(Y\varphi)-(\nabla_X Y)\varphi)-\nabla_Z X(Y\varphi)+(\nabla_{\nabla_Z X} Y)\varphi-X(\nabla_Z Y\varphi)+(\nabla_X \nabla_Z Y)\varphi
\end{align}

thus in the local coordinate, 
\begin{align}
 \nabla^3 \varphi =\varphi_{ijk} dx^i \otimes dx^j \otimes dx^k &\\
 \text{with}\quad \varphi_{ijk}=\partial^3_{ijk} \varphi- \partial_k \Gamma_{ij}^l \partial_l \varphi - \Gamma_{ij}^l \partial^2_{kl}\varphi-\Gamma_{ik}^m\partial^2_{mj} \varphi&+ \Gamma_{ik}^l \Gamma_{lj}^m \partial_m \varphi-\partial_i \Gamma_{jk}^l \partial_l \varphi \nonumber\\
 & -\Gamma_{jk}^l \partial^2_{il}\varphi+ \Gamma_{il}^m\Gamma_{jk}^l\partial_m \varphi
\end{align}
The local coordinate representation is combined with derivative of Christoffel symbol and third order derivative of $\varphi$. And we can prove easily by induction that the local coordinate representation of $\nabla^l \varphi$ is $\varphi_{i_1\cdots i_l} dx^{i_1} \otimes \cdots \otimes dx^{i_l}$
with $\varphi_{i_1\cdots i_l}$ can be represented by derivative of $\varphi$ not higher than order $l$ and derivative of Christoffel symbol not higher than order $l-2$. To prove this, we only need to notice that 
\begin{align}
	\nabla^{l+1}\varphi(X_1,\cdots,X_l,X)&=\nabla(\nabla^l \varphi)(X_1,\cdots,X_l,X) \\
	&=(\nabla_X \nabla^l \varphi) (X_1,\cdots,X_l)\\
	&=X(\nabla^l \varphi(X_1,\cdots,X_l))-\sum_{i=1}^l \nabla^l\varphi(X_1,\cdots,\nabla_X X_i,\cdots,X_l)
\end{align}

Since we assume for any $M \in \mathcal{M}$, the local coordinate
representation in the normal chart of the Christoffel symbols and
their derivatives with order not above $k-2$ have an uniform bound only depends on its order,
and we have claimed derivatives of $\varphi$ which order is not higher
than $k$ have an uniform bound, and we notice that the local
representation of $\nabla^k \varphi$ is combined with them in the same
pattern among all $M \in \mathcal{M}$. Therefore, $\varphi_{i_1\cdots
  i_l}$ is bounded uniformly for all $l \leq k$, i.e.
\begin{align}
|\varphi_{i_1\cdots i_l}| \leq C \lambda_j^\frac{2l+d}{4} \quad \text{for }l\leq k
\end{align}

With this control of upper bound, we can estimate the norm of higher derivatives of $\varphi$:
\begin{align}
	\|d^l \varphi\|&=\sup\limits_{\|v_1\|=1, \dots,\|v_l\|=1} \|\nabla^l \varphi (v_1,\dots,v_l)\| \\
	&=\sup\limits_{\|v_1\|=1, \dots,\|v_l\|=1} \|\nabla^l \varphi (v_1\otimes\dots\otimes v_l)\| \\
	&=|\varphi_{i_1\cdots i_l} dx^{i_1} \otimes \cdots \otimes dx^{i_l} (v_1^{i_1} \frac{\partial}{\partial x^{i_1}} \otimes \cdots \otimes v_l^{i_l} \frac{\partial}{\partial x^{i_l}})|\\
	&=|\varphi_{i_1\cdots i_l} v_1^{i_1} \cdots v_l^{i_l}| \\
	&\leq C \lambda_j^\frac{2l+d}{4} |\sum_{i_1,i_2,\dots i_l=1}^d   v_1^{i_1} \cdots v_l^{i_l}|\\
	&= C \lambda_j^\frac{2l+d}{4} |\sum_{i_1=1}^{d} v_1^{i_1} \sum_{i_2,\dots i_l=1}^dv_2^{i_2} \cdots v_l^{i_l}|\\
	\text{(AM–GM)}&\leq C \lambda_j^\frac{2l+d}{4} |\sum_{i_2,\dots i_l=1}^dv_2^{i_2} \cdots v_l^{i_l}| \sqrt{d}\\
	&\ \ \vdots \\
	&\leq C \lambda_j^\frac{2l+d}{4} d^{l/2}.
\end{align}

This inequality holds for all $l \leq k$. Here $\varphi : M \rightarrow \mathbb{R}$, thus $d^l \varphi_j (v_1,\dots,v_l)$ gives a vector in $\mathbb{R}$, and then $d^l \varphi (v_1,\dots,v_l) = \left(d^l\varphi_1(v_1,\dots,v_l),\dots,d^l\varphi_m (v_1,\dots,v_l)\right)$ is a vector in $\mathbb{R}^m$, which is also a tangent vector in $T_{\varphi(p)}\varphi(M)$ through the natural isometry. Since $d^l \varphi(v_1,\dots,v_l)$ has $m$ components, it is easy to check the operator norm of $d^l \varphi$ has the upper bound $\sqrt{m}C \lambda_m^\frac{2l+d}{4}d^{l/2} = C(l,d,m)\lambda_m^\frac{2l+d}{4}$. Therefore, Since $\lambda_m$ is bounded above by $C(m)$, $\|d^l \varphi\|$ are bounded uniformly for $l \leq k$.

\subsection{Uniform upper bounds  $\left\{L_{\perp},L_3,\dots,L_k\right\}$ on the higher derivatives of the local parametrization reminder}
\label{sec:phim-L}

In this section, we will verify that there exists an common parameter
set
$\textbf{L}^\prime:=\left\{L_{\perp}^\prime,L_3^\prime,\dots,L_k^\prime\right\}$
such that the operator norm of the derivatives of remainder $N_{\varphi(p)}$ defined in Assumption \ref{as:smooth} can be controlled uniformly over $\varphi(\MM)$.

We fix a point $p^\prime=\varphi(p) \in \varphi(M)$, for $v^\prime \in T_{p^\prime}\varphi(M)=d\varphi_p(T_pM)$, we define 
\begin{align}
	\Psi^\prime_{p^\prime}(v^\prime)=\varphi(\Psi_p(d(\varphi^{-1})_{p^\prime} (v^\prime)))=p^\prime+v^\prime+N^\prime_{p^\prime}(v^\prime)
\end{align}
from definition of $\Psi_p$, we know $\Psi^\prime_{p^\prime}(v^\prime)$ and $N^\prime_{p^\prime}(v^\prime)$ are well defined when $\Vert d(\varphi^{-1})_{p^\prime}(v^\prime)\Vert \leq \frac{1}{4L_\perp}$. 

We estimate $\|d \varphi^{-1}_{p^\prime}\|$: We can treat $T_p M$ as a
$d$ dimensional subspace of $\mathbb{R}^D$ and treat $T_{p^\prime}
\varphi(M)$ is a  $d$ dimensional subspace of $\mathbb{R}^m$, and these
inclusions are canonical and isometric, so they are both closed
subspaces. And $d \varphi_p$ is homeomorphism, so graph of $d
\varphi^{-1}_{p^\prime}$ is closed, then $d \varphi^{-1}_{p^\prime}$
is bounded, i.e. $\|d \varphi^{-1}_{p^\prime}\|$ exists.
\begin{align}
	\|d \varphi^{-1}_{p^\prime}\|&=\sup \left\{\|d\varphi^{-1}_{p^\prime} v^\prime \|: v^\prime \in Im(d\varphi_p), \|v^\prime\| = 1\right\} \\
	&=\sup \left\{\frac{\|v\|}{\|d\varphi_p v\|}: v^\prime \in Im(d\varphi_p), \|v^\prime\| = 1, v=d \varphi^{-1}_{p^\prime} v^\prime\right\} \\
	&=\sup \left\{\frac{1}{\|d\varphi_p v\|}: \|v\| = 1\right\} \\
	&\leq \frac{1}{1-\epsilon}
\end{align}
so when $\|v^\prime\| \leq \frac{1}{4L_\perp(1-\epsilon)}$, $\Vert d(\varphi^{-1})_{p^\prime}(v^\prime)\Vert \leq \frac{1}{4L_\perp}$, then $\Psi^\prime_{p^\prime}(v^\prime)$ and $N^\prime_{p^\prime}(v^\prime)$ are well defined.

Now, We check $N^\prime_{p^\prime}(0)=0$ and $d_0N^\prime_{p^\prime}=0$:
\begin{align}
	N^\prime_{p^\prime}(0)+p^\prime=\varphi(\Psi_p(d(\varphi^{-1})_{p^\prime}(0)))=\varphi(\Psi_p(0))=\varphi(p+0+N_p(0))=\varphi(p)=p^\prime\\
	d_0N^\prime_{p^\prime}=d\varphi_p d_0\Psi_p d(\varphi)_{p^\prime}^{-1}-I_m=d\varphi_p(I_D+d_0N_p) d(\varphi)_{p^\prime}^{-1}-I_m=0
\end{align}

We assume $\|v^\prime\| \leq \frac{1}{4L_\perp(1-\epsilon)}$ and denote $d\varphi_{p^\prime}^{-1}v^\prime$ as $v$, and $c(v^\prime)=p+v+N_p(v)$, then for any unit vector $w^\prime \in T_{p^\prime} \varphi(\mathcal{M})$, using Faà di Bruno's formula, we have
\begin{align}
	\|d^k_{v^\prime}N^\prime_{p^\prime}(w^{\prime \otimes k})\|&=\|\sum\limits_{\pi \in P(k)} d^l_{c(v^\prime)} \varphi \circ \left(d^{j_1}_v \Psi_p \left(\left\{d_{p^\prime}\varphi^{-1} w^\prime\right\}^{\otimes j_{1}}\right),\dots,d^{j_l}_v \Psi_p\left(\left\{d_{p^\prime}\varphi^{-1} w^\prime\right\}^{\otimes j_{l}}\right) \right)\|
\end{align}
where $P(k)$ is a partition of $k$ with $l$ parts such that $j_1+\dots+j_l=k$.

Since $\Psi_p(v)=p+v+N_p(v)$, so $d_v \Psi_p = I+ d_v N_p$, $d^i_v \Psi_p = d^i_v N_p$, thus
\begin{align}
	\|d_v \Psi_p\| = \| I+ d_v N_p \| \leq 1 + L_\perp \|v\| \\
	\|d^i_v \Psi_p\| =\|d^i_v N_p\| \leq L_i
\end{align}
therefore
\begin{align}
\|d^k_{v^\prime}N^\prime_{p^\prime}(w^{\prime \otimes k})\| &\leq \sum\limits_{\pi \in P(k)} \|d^l \varphi\| \prod_{i=1}^l \| d^{j_i} \Psi_p\| \|d_{p^\prime}\varphi^{-1} w^\prime\|^{j_i} \\
&\leq \sum\limits_{\pi \in P(k)} \|d^l \varphi\| \frac{1}{(1-\epsilon)^l} \prod_{i=1}^l \| d^{j_i} \Psi_p\| 
\end{align}

We have proved $\|d^j \varphi\|$ and $\|d^j_v \Psi_p\|$ have the uniform upper bound for $j=2,\dots,k$, so we can select $\textbf{L}^\prime:=\left(L_\perp^\prime,\dots,L_k^\prime\right)$ as the uniform upper bound of $\|d^2_{v^\prime}N^\prime_{p^\prime}\|,\dots,\|d^k_{v^\prime}N^\prime_{p^\prime}\|$.

\subsection{Bounds on the pushforward density}
\label{sec:density}

In this subsection, we estimate the pushforward density of the sampling process. By establishing both lower and upper bounds for the pushforward density, we ensure that the sampling over $\varphi(M)$ remains approximately uniform. This guarantees that there are always some sampling points in any nonzero measure region with nonzero probability, thereby allowing the manifold to be effectively approximated using the sampled points.

We denote $d\varphi_p:T_p M \rightarrow T_{p^\prime} N$ as $A$, then for $\|v\|=1$, $1-\epsilon<\|Av\|<1+\epsilon$, which means $\|A\|<1+\epsilon$, so $\sqrt{\lambda_{\max}(A^T A)}<1+\epsilon$. We know $A^T A$ is positive-definite  symmetric matrix, we consider an unit eigenvector $v$ of it with eigenvalue $\lambda$, then 
\begin{align}
	\lambda=\|v^T \lambda v\|= \|v^T A^T A v\|=\|Av\|^2\in ((1-\epsilon)^2,(1+\epsilon)^2)
\end{align}

Since $A^TA$ is symmetric, then $A^T A=Q\Lambda Q^T$,  where $Q$ is orthogonal matrix and $\Lambda$ is diagonal matrix whose diagonal elements are eigenvalues of $A^T A$ and all of them are close to $1$.
\begin{align}
	|\|\pi_{T_p M} A^T A v\|-\|v\||&\leq \|\pi_{T_p M}(A^T A -I) v\|\\
	&=\|\pi_{T_p M} Q(\Lambda-I)Q^T v\|\\
	&\leq \|Q(\Lambda-I)Q^T v\|\\
	&\leq 3\epsilon
\end{align}
the last inequality is due to orthogonal matrix does not change norm of vector and diagonal elements in $\Lambda-I$ are smaller than $3\epsilon$.

Thus we have $1-3\epsilon <\|\pi_{T_p M} A^T A v\|<1+3\epsilon $, so the eigenvalues of $\pi_{T_p M} A^T A|_T$ is greater than $1-3\epsilon$ and less than $1+3\epsilon$. Therefore, if 
\begin{align}
	\label{range_of_e}
	\epsilon \leq \min\left\{\frac{4^{1/d}-1}{3},\frac{1-\frac{1}{4^{d/1}}}{3}\right\},
\end{align}
then
\begin{align}
	\det (\pi_{T_p M} A^T A|_T)=\prod \lambda_i\in ((1-3\epsilon)^d,(1+3\epsilon)^d)\subset(\frac{1}{4},4),
\end{align}
which implies
\begin{align}
	\sqrt{\det \left(\pi_{T_p M} \circ d\varphi_p^T \circ d\varphi_p |_{T_pM}\right)} \in (\frac{1}{2},2).
\end{align}

According to \cref{thm:push_density}, we have the pushforward density of $\varphi_\# P$ is 
	\begin{align}
	g(p^\prime)=f(p)/\sqrt{\det \left(\pi_{T_p M} \circ d\varphi_p^T \circ d\varphi_p |_{T_pM}\right)},
\end{align}
where $p^\prime = \varphi(p)$. Therefore, $\frac{f_{min}}{2}<g<2f_{\max}$.

In our setting, our sampleing is uniform, thus $f \equiv \frac{1}{Vol(M)}$, which is bounded from both and above, which implies the density of $\varphi(\MM)$ is also bounded from below by $\min\limits_{M \in \MM} \left\{Vol(M)\right\}/2$ and bounded above by $2\max\limits_{M \in \MM} \left\{Vol(M)\right\}$

\subsection{Estimation of $\tau_{min}$}
\label{sec:reach}
    
    If $M$ has reach $\tau$, then at least one
  of the following cases holds:
\begin{enumerate}
	\item[I] (Global case) $M$ has a bottleneck, i.e. there exist $p, q \in M$, such that $(p+q)/2 \in \text{Med}(M)$ and $\|p-q\|=2\tau$.
	\item[II] (Local case) There exists $p \in M$, and an arc-length parametrized geodesic $\gamma$ such that $\gamma(0)=p$ and $|\gamma^{\prime\prime}(0)|=1/\tau$.
\end{enumerate}
where $\text{Med}(M)=\left\{z\in \mathbb{R}^D: \exists p\neq q\in M, \|z-p\|=\|z-q\|=d_E(z,M) \right\}$, $d_E(x,M)$ is the distance between $x$ and $M$.

We denote global reach and local reach as $\tau_g$ and $\tau_l$, respectively. Thus $\tau = \min \left\{\tau_g, \tau_l\right\}$

\begin{figure}[H]
	\centering
\begin{tikzpicture}
	
	\draw[thick] plot [smooth cycle, tension=0.9] coordinates {(-4.5,-1.5) (-3.8,1.5) (-1.5,0.5) (0.5,0.5) (3,1.5) (4,-1.4) (0,-0.5) };
	\draw[dashed] (3.65,-0.9) circle [radius=0.59];
	\fill (3.65,-0.9) circle [radius=0.03];
	\draw[<->, dashed, blue] (3.65,-0.9) -- (4,-1.4);
	
	\node[blue]  at (4.0,-1.1) {\(\tau_l\)};
	\draw[dashed, thick] plot [smooth, tension=0.8] coordinates {(-3.2,0) (-1,-0.1) (0,-0.05) (1,0.05) (2.5,0.2)};
	
	\draw[dashed, thick] (-3.2,0) -- (-4,-1);
	\draw[dashed, thick] (-3.2,0) -- (-3.5,1);
	\draw[dashed, thick] (2.5,0.2) -- (3.4,-0.8);
	\draw[dashed, thick] (2.5,0.2) -- (2.8,0.8);
	
	\draw[<->, dashed, black] (-0.32,-0.07) -- (-0.3,-0.5);
	
	\node[black] at (0.2,-0.25) {\(\tau_g\)};
	\node at (0.3,1.2) {\(M\)};
	\node at (-2,0.3) {\(Med(M)\)};
	
\end{tikzpicture}
	\caption{Reach}
\end{figure}

And we can also define reach $\tau:=\min \left\{\tau_l,\tau_{\wfs}\right\}$, where $\tau_l$ is local reach and $\tau_{\wfs}$ is weak feature size, and we give more details about them. Let $\Gamma_M(y)=\left\{x\in M: d_E(y,M)=|x-y| \right\}$, then define generalized gradient:
\begin{align}
	\nabla_M(y):=\frac{y- \text{Center}(\Gamma_M(y))}{d_E(y,M)}
\end{align}
where $\text{Center}(A)$ is the center of the smallest ball enclosing the bounded subset $A \subset \mathbb{R}^D$. We say $y$ is a critical point of $d_E(\cdot,M)$ if $\nabla_M(y)=0$, then we can define \begin{align}
	\tau_{\wfs}:=\inf \left\{d_E(y,M), y \in \mathcal{C}\right\}
\end{align}where $\mathcal{C}$ is the set of critical points.

And we can define $\tau_l$ easily by 
\begin{align}
	\tau_l := \inf\limits_{p\in M} \left\{\frac{1}{\|\Rmnum{2}_p\|}\right\}
\end{align}

\subsubsection{Local Reach}
Local reach is a quantity measuring the "curvature" locally, controlling of local reach will avoid some extremely weird manifold. Now we estimate the lower bound of local reach for $\varphi(\mathcal{M})$. 

To control the local reach, we will derive the local coordinate representation of geodesic and then use the boundedness of derivative of eigenfunctions to obtain the upper bound of geodesic of $\varphi(\mathcal{M})$ in $\mathbb{R}^m$.

 For any $p \in M$, we can choose neighborhoods $U\subset M$ and $V \subset \varphi(M)$ s.t. $\varphi(U)=V$, and $(U,\mathbf{x})$ is a normal coordinate chart of $p$, where $\mathbf{x}^{-1}=\exp_p \circ E$, $E$ is the isomorphism from $\mathbb{R}^d$ to $T_pM$, and we denote $\mathbf{x}(U)$ as $\hat{U} \subset \mathbb{R}^d$ is open. 

We have a parametrization of $U\subset M$: 
\begin{align}
	\mathbf{x}^{-1}: \hat{U}\subset \mathbb{R}^d &\rightarrow U \subset M \\
	(x_1,\dots,x_d) &\rightarrow \mathbf{x}^{-1}(x_1,\dots,x_d)= (\mathbf{x}^{-1}_1(x_1,\dots,x_d),\dots,\mathbf{x}^{-1}_D(x_1,\dots,x_d))
\end{align}

Using this map, we have a coordinate representation of $\varphi$, i.e.  $\hat{\varphi}:=\varphi \circ \mathbf{x}^{-1}$ of $V \subset \varphi(M)$ by $\hat{U}\overset{\mathbf{x}^{-1}}{\rightarrow}U \overset{\varphi}{\rightarrow} V \subset \mathbb{R}^m$:
\begin{align}
(x_1,\dots,x_d) &\overset{\mathbf{x}^{-1}}{\rightarrow} \mathbf{x}^{-1}(x_1,\dots,x_d) \overset{\varphi}{\rightarrow} (\varphi_1\circ \mathbf{x}^{-1},\dots,\varphi_m\circ \mathbf{x}^{-1})(x_1,\dots,x_d)
\end{align}

Let $r=\hat{\varphi}$ be the parameterization of $\varphi(M)$, we compute its $k$-th tangent vector
\begin{align}
	r_i=\frac{\partial \hat{\varphi}}{\partial x_i}=\left(\frac{\partial \hat{\varphi}_1}{\partial x_i},\dots,\frac{\partial \hat{\varphi}_m}{\partial x_i}\right)\quad  i=1,\dots,d,
\end{align}
and $r$'s second derivative is 
\begin{align}
	r_{ij}=\left(\frac{\partial^2 \hat{\varphi}_1}{\partial x_i \partial x_j},\dots,\frac{\partial^2 \hat{\varphi}_m}{\partial x_i \partial x_j}\right)\quad  i,j=1,\dots,d .
\end{align}

For $\varphi(p) \in \varphi(M)$, we have the orthogonal decomposition:
\begin{align}
	\R^m=T_{\varphi(p)}\varphi(M) \oplus N_{\varphi(p)}\varphi(M),
\end{align}
where $T_{\varphi(p)}\varphi(M)$ is spanned by $\left\{r_1,\dots,r_d\right\}$ and $ N_{\varphi(p)}\varphi(M)$ is the $m-d$ subspace orthogonal to $T_{\varphi(p)}\varphi(M)$. Let $\pi$ be the orthogonal projection from $\R^m$ onto $N_{\varphi(p)}\varphi(M)$, we can derive the second fundamental form on $\varphi(M)$, which is a symmetric 2-tensor field given by 
\begin{align}
	\Rmnum{2}_{ij} dx^i \otimes dx^j \quad 
	\text{with} \quad \Rmnum{2}_{ij} = \pi(r_{ij}) \in  N_{\varphi(p)}\varphi(M).
\end{align}

Then we can verify that 
\begin{align}
	\|\Rmnum{2}\|_{op} \leq d \max \left\{\pi(r_{ij})\right\}\leq d \max\left\{|r_{ij}|\right\}.
\end{align} 

We compute the module of $r_{ij}=\left(\frac{\partial^2 \hat{\varphi}_1}{\partial x_i \partial x_j},\dots,\frac{\partial^2 \hat{\varphi}_m}{\partial x_i \partial x_j}\right)$ in more detail. We have $\hat{\varphi}_i = (t_0)^\frac{d+2}{4}\sqrt{2}(4\pi)^\frac{d}{4} e^{-\lambda_i t_0/2} \hat{e}_i$, thus

\begin{align}
	\left|\frac{\partial^2 \hat{\varphi}_k}{\partial x_i \partial x_j}\right|=(t_0)^\frac{d+2}{4}\sqrt{2}(4\pi)^\frac{d}{4} e^{-\lambda_k t_0/2} \left|\frac{\partial^2 \hat{e}_k}{\partial x_i \partial x_j}\right| \\
	\leq (t_0)^\frac{d+2}{4}\sqrt{2}(4\pi)^\frac{d}{4} e^{-\lambda_k t_0/2} C_2 (1+\lambda_k)^{1+\frac{d}{4}},
\end{align}
where the inequality is from \cref{deriinequ} in \cref{sec:phim-higher-deri}.

We also have the lower bound of $\lambda_k$, 
\begin{align}
\lambda_k \geq C_1^{1+\text{diam}\sqrt{\kappa}} \text{diam}^{-2} k^{2/d},
\end{align}
where $C_1$ is constant only depending on $d$.\cite{hassannezhad2016eigenvalue}.

Therefore, 
\begin{align}
|r_{ij}|\leq (t_0)^\frac{d+2}{4}\sqrt{2}(4\pi)^\frac{d}{4} C_2\sqrt{\sum_{k=1}^{m}e^{-\lambda_k t_0} (1+\lambda_k)^{2+\frac{d}{2}}}.
\end{align}

It is easy to see that $e^{-x t_0} (1+x)^{4+d}$ is increasing when $x<\frac{4+d}{2t_0}-1$ and decreasing when $x>\frac{4+d}{2t_0}-1$, thus we can use 
\begin{align}
C_1^{1+\text{diam}\sqrt{\kappa}} \text{diam}^{-2} k^{2/d} \leq \lambda_k \leq C+C^\prime k^{2/d},
\end{align}
which implies $c_1 k^{2/d} \leq \lambda_k \leq c_2 k^{2/d}$, where $c_1,c_2$ are constants depending on $\kappa$, $\text{diam}$, $V$, $d$. Then we have $|r_{ij}|$ is bounded from above, thus $\|\Rmnum{2}\|_{op}$ is bounded from above, which implies that $\tau_{l,\varphi(M)}$ is bounded from below.

\subsubsection{Global Reach}
\label{sec:glo-reach}
In this subsection, we estimate the lower bound of the global reach of $\varphi(\MM)$. The global reach constrains the overall shape of the manifold by controlling the Euclidean distance to ''separate'' different parts of the manifold.

Under Assumptions \ref{as:reach}, when $p$ is close to $q$, the geodesic
distance $d(p,q)$ can be controlled by the Euclidean distance $|p-q|$. 

\begin{lemma}
	\label{lem:geodis_eucdis}
	If $d(p,q)=s$, then
	\begin{align}
		s-\frac{s^3}{24 r_0^2}\leq |p-q| \leq s,
	\end{align}
	where $\frac{1}{r_0}=\sup
	\left\{|\gamma^{\prime\prime}(s)|\right\}$ and $\gamma$ varies among
	all geodesics on $M$ in arc length parameter.
\end{lemma}

We consider the left inequality of \cref{lem:geodis_eucdis}, it is easy to see if $s\leq 2\sqrt{2}r_0$, then
\begin{align}
	\frac{2}{3}s\leq s-\frac{s^3}{24r_0^2} \leq |p-q|,
\end{align}
which implies 
\begin{align}
	|p-q|\leq d(p,q) \leq \frac{3}{2}|p-q|.
\end{align}

We can verify easily that $1/r_0=\sup
\left\{|\gamma^{\prime\prime}(s)|\right\}=1/\tau_{l}$, and this local linear approximation can exclude the global reach case when $d(p,q)< 2\sqrt{2}\tau_{l}$, and $\frac{3}{2}$ plays an important role here, it will lead to the contradiction to the global reach case.

\begin{lemma}
	\label{global_case}
	If $d(p,q)\leq s_0$ where $s_0=2\sqrt{2}\tau_{l}$, then $p,q$ cannot satisfy global reach case.
\end{lemma}

Since $\varphi$ is almost isometry, i.e. $| \|d\varphi\|-1| < \epsilon$, thus $d(p,q)\leq s/1+\epsilon$ implies $d(\varphi(p),\varphi(q)) \leq s$. We select $s_0 = 2\sqrt{2}\tau_{l,\varphi(M)}$, then $\varphi(p),\varphi(q)$ cannot satisfy global reach case.

 Since we have proved that for $\varphi(M) \in \varphi(\mathcal{M})$, there exists an uniform lower bound of local reach for them which only depends on the geometric properties of $\MM$, we denote this lower bound as $\tau_{l,\varphi(\MM)}$, then we choose $s_0 = 2\sqrt{2}\tau_{l,\varphi(\MM)}$ such that $d(p,q) \leq s_0/1+\epsilon$, then $\varphi(p), \varphi(q)$ cannot satisfy global reach case.

Consequently, only when $d(p,q)>\frac{s_0}{1+\epsilon}$, $\varphi(p),\varphi(q)$ can satisfy global reach case. And in the following, we will claim that if the geodesic distance of two points $p,q \in M$ is large enough, then the Euclidean distance of $\varphi(p), \varphi(q) \in \varphi(M)$ is bounded from below, i.e. 
\begin{theorem}
	\label{global condition}
	For large enough $s_1>0$, there exists $r_1>0$ such that
	\begin{align}
		d(p,q)> s_1 \implies |\varphi(p)-\varphi(q)|>r_1.
	\end{align}
	We call this the {\em global reach condition}.
\end{theorem}

To bound the Euclidean distance from below, we need the estimate of heat kernel:
\paragraph{Estimate for Heat Kernel}
We will bound the heat kernel $K(t,p,q)$ above and $K(t,p,p)$ from below to control this Euclidean distance. For upper bound, we have
\begin{lemma}
	\label{thm:upperbound}
	Let M be a complete Riemannian manifold of dimension $d$ with Ricci curvature is greater than $-\kappa(d-1)$ for some $\kappa\geq 0$, then heat kernel satisfies:
	\begin{align}
		K(t,p,q)\leq \frac{C_1(M)}{t^{d/2}}\exp\left(C_2(M)\kappa t - \frac{2d(p,q)^2}{9t}\right)
	\end{align}.
\end{lemma}

The existence of $C_1(M)$ and $C_2(M)$ follows from \cref{thm:upper_app} of Li-Yau and \cref{injvol} of Croke with $\alpha_1=\frac{3}{2}$, $\alpha_2=\frac{1}{2}$ and $t_0\leq \iota^2/4$, and \cref{thm:upper_app}, \cref{injvol} are in the appendix.

For lower bound, we only need on-diagonal lower estimate:
\begin{lemma}
	\label{thm:lowerbound}
	Let M be a complete Riemannian manifold of dimension $d$ with Ricci curvature is greater than $-\kappa(d-1)$ for some $\kappa\geq 0$. For any $t>0$ and $p\in M$, we have
	\begin{align}
		K(t,p,p) \geq (4\pi t)^{-d/2} \exp\left(-\frac{\beta^2}{4}t-\frac{2\sqrt{3d}\beta}{3}\sqrt{t}\right).
	\end{align}
	where $\beta=\sqrt{\kappa}(d-1)$.\cite{wang1997sharp}
\end{lemma}
This follows \cref{thm:lower_app} of Wang with $\sigma^2=\frac{3\beta^2}{8d}$ and $p=q$.

\begin{proof}[Proof of \cref{global condition}]
To show this, we will compute the distance between $\varphi(p)$ and $\varphi(q)$ in $\mathbb{R}^m$ directly and then estimate it using geodesic distance of $p,q$.

For any $p \in M$, $\varphi(p)=(t_0)^\frac{d+2}{4}\sqrt{2}(4\pi)^\frac{d}{4}\left(e^{-\lambda_1 t_0/2}e_1(p),\cdots,e^{-\lambda_m t_0/2}e_m(p)\right) \in \mathbb{R}^m$, thus the Euclidean distance between $\varphi(p)$ and $\varphi(q)$ is
\begin{align}
	|\varphi(p)-\varphi(q)|&=(t_0)^\frac{d+2}{4}\sqrt{2}(4\pi)^\frac{d}{4} \sqrt{\sum_{i=1}^me^{-\lambda_i t_0}(e_i(p)-e_i(q))^2}\\
	&=(t_0)^\frac{d+2}{4}\sqrt{2}(4\pi)^\frac{d}{4}\sqrt{\sum_{i=1}^me^{-\lambda_i t_0}(e_i(p)^2+e_i(q)^2-2e_i(p)e_i(q))} \\
	&=(t_0)^\frac{d+2}{4}\sqrt{2}(4\pi)^\frac{d}{4}\sqrt{\sum_{i=0}^m(e^{-\lambda_i t_0}e_i(p)^2+e^{-\lambda_i t_0}e_i(q)^2-2e^{-\lambda_i t_0}e_i(p)e_i(q))},
\end{align}
the last equality holds since $e_0$ is a constant function.

We have the representation of heat kernel
\begin{align}
	K(t,p,q)=\sum_{i=0}^{\infty} e^{-\lambda_i t} e_i(p) e_i(q),
\end{align}
and the truncated heat kernel $K_m$ is
\begin{align}
	K_m(t,p,q)=\sum_{i=0}^{m} e^{-\lambda_i t} e_i(p) e_i(q),
\end{align}
and we have $\|K_m(t_0,p,\cdot)-K(t_0,p,\cdot)\|_{\infty} \leq \epsilon^\prime$,
thus 
\begin{align}
	|\varphi(p)-\varphi(q)|&\geq (t_0)^\frac{d+2}{4}\sqrt{2}(4\pi)^\frac{d}{4} \sqrt{K_m(t_0,p,p)+K_m(t_0,q,q)-2K_m(t_0,p,q)}\\
	&\geq (t_0)^\frac{d+2}{4}\sqrt{2}(4\pi)^\frac{d}{4} \sqrt{K(t_0,p,p)+K(t_0,q,q)-2 K(t_0,p,q)-4\epsilon^\prime}.
\end{align}

Combining two inequities in \cref{thm:upperbound} and \cref{thm:lowerbound}, we have
\begin{align}
	|\varphi(p)&-\varphi(q)|\geq (t_0)^\frac{d+2}{4}2(4\pi)^\frac{d}{4}  \nonumber \\ \times &\sqrt{(4\pi t_0)^{-\frac{d}{2}} \exp\left(-\frac{\beta^2t_0}{4}-\frac{2\sqrt{3dt_0}\beta}{3}\right) - \frac{C_1(\MM)}{t_0^{d/2}}\exp\left(C_2(\MM)\kappa t_0 - \frac{2d(p,q)^2}{9t_0}\right) - 2\epsilon^\prime}. \label{sqrt_dis}
\end{align}
In our case, $t_0, m$ are fixed when embedding $M$ into $\mathbb{R}^m$. The first two items in the square root
are only depends on the geometric properties of $\MM$, and $\epsilon^\prime$ is chosen when we perform DM, we can choose appropriate $\epsilon^{\prime}$ based on our geometric setting such that 
\begin{align}
2\epsilon^{\prime}<(4\pi t_0)^{-\frac{d}{2}} \exp\left(-\frac{\beta^2t_0}{4}-\frac{2\sqrt{3dt_0}\beta}{3}\right)
\end{align}

 Therefore, for large enough $s_1$ with $d(p,q)\geq s_1$, we have $|\varphi(p)-\varphi(q)|\geq r_1$, where $r_1, s_1$ depend on our
setting of $\MM$.
\end{proof}

We calculate more carefully to select the appropriate $s_1$ and $\epsilon^\prime$ and derive $r_1$, we consider the item under square root in \cref{sqrt_dis} which is positive, that is
\begin{align}
&(4\pi t_0)^{-\frac{d}{2}} \exp\left(-\frac{\beta^2t_0}{4}-\frac{2\sqrt{3dt_0}\beta}{3}\right) - \frac{C_1(\MM)}{t_0^{d/2}}\exp\left(C_2(\MM)\kappa t_0 - \frac{2s_1^2}{9t_0}\right) - 2\epsilon^\prime > 0 \\
\iff&  \exp\left(C_2(\MM)\kappa t_0 - \frac{2s_1^2}{9t_0}\right)< \frac{(4\pi)^{-\frac{d}{2}}}{C_1(\MM)} \exp\left(-\frac{\beta^2t_0}{4}-\frac{2\sqrt{3dt_0}\beta}{3}\right) - \frac{2(t_0)^\frac{d}{2}\epsilon^\prime}{C_1(\MM)}.
\end{align}

We denote the right side $\frac{(4\pi)^{-\frac{d}{2}}}{C_1(\MM)} \exp\left(-\frac{\beta^2t_0}{4}-\frac{2\sqrt{3dt_0}\beta}{3}\right) - \frac{2(t_0)^\frac{d}{2}\epsilon^\prime}{C_1(\MM)}$ as $F$, we need it be a fixed constant, thus we can select 
\begin{align}
	\label{sele eprime}
\epsilon^\prime=\frac{(4\pi t_0)^{-\frac{d}{2}}}{8}\exp\left(-\frac{\beta^2t_0}{4}-\frac{2\sqrt{3dt_0}\beta}{3}\right),
\end{align}
which implies
\begin{align}
	F= \frac{3(4\pi)^{-\frac{d}{2}}}{4C_1(\MM)} \exp\left(-\frac{\beta^2t_0}{4}-\frac{2\sqrt{3dt_0}\beta}{3}\right),
\end{align}
Therefore, $\epsilon^\prime$ is only depending on the geometric properties of $\MM$. Choosing $\epsilon^\prime$ will affect $m=\max\left\{N_0(d,\kappa,\iota,\epsilon,V,\frac{t_0}{2}),N_1(d,\kappa,\iota,V,\epsilon^\prime,t_0)\right\}+1$ and is independent with choice of $t_0=t_0(d,\kappa,\iota,\epsilon)$ when perform DM. We need set appropriate $\epsilon^\prime$ based on geometric setting and derive $m$ and $t_0$ at the beginning, but until now we know how to choose them.

Then we ask $\exp\left(C_2(\MM)\kappa t_0 - \frac{2s_1^2}{9t_0}\right) < \frac{2}{3}F$, which is equivalent to
\begin{align}
	s_1^2>\frac{9t_0}{2}\left(C_2(d)\kappa t_0 + \frac{\beta^2 t_0}{4}+\frac{2\sqrt{3d t_0}\beta}{3}+ \log\left(2(4\pi)^\frac{d}{2}C_1(d)\right)\right), \label{smallests1}
\end{align}
Intuitively, $t_0$ is small, so $s_1$ is not very large.

When $s_1$ is the square root of \cref{smallests1} and  $\epsilon^\prime=\frac{(4\pi t_0)^{-\frac{d}{2}}}{8}\exp\left(-\frac{\beta^2t_0}{4}-\frac{2\sqrt{3dt_0}\beta}{3}\right)$, we have 
\begin{align}
	r=r_1=\sqrt{t_0} \exp\left(-\frac{\beta^2t_0}{8}-\frac{\sqrt{3dt_0}\beta}{3}\right)
\end{align}

Now we consider two cases based on $d(p,q)$. The first case is $d(p,q) \leq s_0/1+\epsilon$, then $d(\varphi(p),\varphi(q))\leq s_0$ since $|\|d\varphi\|-1|<\epsilon$, which implies $\varphi(p), \varphi(q)$ cannot be the global reach case. The second case is $d(p,q)>s_1$, then $|\varphi(p)-\varphi(q)|> r_1$. To make there is no gap between $s_0/1+\epsilon$ and $s_1$, we need 
\begin{align}
	s_0^2>\frac{9(1+\epsilon)^2t_0}{2}\left(C_2(\MM)\kappa t_0 + \frac{\beta^2 t_0}{4}+\frac{2\sqrt{3d t_0}\beta}{3}+\log\left(2(4\pi)^\frac{d}{2}C_1(\MM)\right)\right).
\end{align}
Combining $s_0=2\sqrt{2}r_0$ and $r_0 =\tau_{l,\varphi(\MM)}$. The condition making $s_0/1+\epsilon \geq s_1$ is 
\begin{align}
	\label{star}
	8\tau^2_{l,\varphi(\MM)} 
	\geq \frac{9(1+\epsilon)^2t_0}{2}\left(C_2(\MM)\kappa t_0 + \frac{\beta^2 t_0}{4}+\frac{2\sqrt{3d t_0}\beta}{3}+ \log\left(2(4\pi)^\frac{d}{2}C_1(\MM)\right)\right), \tag{*}
\end{align}
which is an assumption only based on geometric properties since $t_0=t_0(d,\kappa,\iota,\epsilon)$, $m=\max\left\{N_0,N_1\right\}+1$ and so on.

We assume that $\mathcal{M}$ satisfies regularity conditions such
that $\cref{star}$ holds, then $\|\varphi(p)-\varphi(q)\|\geq r_1$ for $d(p,q)>s_0/1+\epsilon$, and we have claimed that $\varphi(p),\varphi(q)$ with $d(p,q)\leq s_0/1+\epsilon$ cannot be the global reach case, thus $\tau_{g} \geq \frac{r_1}{2}$.

Therefore, the local reach and global reach both have the uniform lower bound, thus $\tau_{\varphi(M)}\geq \tau_{min,\mathcal{M}}$ for some fixed constant.

	\section{Convergence of eigenfunctions and eigenvectors/Finite sample error of Diffusion Map $\varphi$}
        \label{sec:convergence}

We consider $n$ sample points $\mathcal{X}_n=\left\{x_1,\dots,x_n\right\}\subset M \in \mathcal{M}$, we define graph affinity matrix $W$ and the degree matrix $D$ as:
\begin{align}
	W_{ij}=\frac{k_h (x_i,x_j)}{q_h(x_i)q_h(x_j)}, \quad D_{ii}=\sum_{j=1}^n W_{ij},
\end{align}
where $k_h(x,y)=\exp \left(-\frac{|x-y|^2}{4h^2}\right)$ is the Gaussian kernel and $q_h(x)=\sum_{i=1}^n k_h(x,x_i)$.

Then the normalized graph Laplacian $L_n$ is defined as
\begin{align}
	\label{Lap matrix}
L_n=\frac{D^{-1}W-I}{h^2}.
\end{align} 

We denote its $i$-th eigenvalue of $-L_n$ as $\mu_{i,n,h}$ with corresponding eigenvector $\tilde{v}_{i,n,h}$ normalized in $l^2$ norm. It is easy to verify that $\mu_{0,n,h}=0$ and $\tilde{v}_{0,n,h} = \mathbbm{1}$. Let $\mathbb{N}(i)=|B_h(x_i) \cap \left\{x_1,\dots,x_n\right\}|$ which is the cardinal of points in the $h$-ball of $x_i$. Then we define the $l^2$ norm of $\tilde{v}$ with respect to inverse estimate density $1/\hat{p}$ as:
\begin{align}
	\|\tilde{v}\|_{l^2(1/\hat{p})} = \sqrt{\frac{\omega(d-1)h^d}{d}\sum_{i=1}^{n}\frac{\tilde{v}^2(i)}{\mathbb{N}(i)}},
\end{align}
and we define
\begin{align}
	v_{i,n,h}=\frac{\tilde{v}_{i,n,h}}{\|\tilde{v}_{i,n,h}\|_{l^2(1/\hat{p})}}.
\end{align}

Let $\Delta$ be the Laplace-Beltrami operator of $M$, and $0=\lambda_0<\lambda_1\leq \lambda_2 \leq \cdots$ be the eigenvalues of $-\Delta$. Denote $e_i$ be the eigenfunctions of $-\Delta$ corresponding $\lambda_i$. Then we have
\begin{theorem}
	\label{thm:conver_rate}
Let $M$ be a $d$-dimensional smooth, closed and connected Riemann manifold embedded in $\R^D$, $f$ be the smooth probability density function on $M$ with infimum $f_{min}>0$ and $\mathcal{X}_n=\left\{x_1, \dots,x_n\right\}$ be the point cloud sampled following $f$ independently and identically. Suppose eigenvalues of $\Delta$ are simple. For fixed $m\in \N$, denote $\Gamma_m=\min\limits_{1\leq i \leq m} \text{dist}(\lambda_i, \sigma(-\Delta)\textbackslash\left\{\lambda_i\right\})$, where $\sigma(-\Delta)$ is the spectrum of $-\Delta$. Suppose 
\begin{align}
	h \leq \mathcal{K}_1 \min \left(\left(\frac{\min(\mathsf\Gamma_m,1)}{\mathcal{K}_2+\lambda_m^{d/2+5}}\right)^2,\, \frac{1}{(\mathcal{K}_3+\lambda_m^{(5d+7)/4})^2}\right),
\end{align}
where $\mathcal{K}_1$ and $\mathcal{K}_2, \mathcal{K}_3>1$  are constants depending on $d$, $f_{\min}$, $\|f\|_{C^2}$, and the volume, the injectivity radius, the curvature and the second fundamental form of the manifold.
Then, when $n$ is sufficiently large so that $h=h(n) \geq (\frac{\log n}{n})^{\frac{1}{4d+13}}$, with probability greater than $1-n^{-2}$, for all $0 \leq i < m$, 
\begin{align}
	|\mu_{i,n,h}-\lambda_{i}|\leq \Omega_1 h^{3/2}. 
\end{align}
And when $n$ is sufficiently large so that $h=h(n)\geq (\frac{\log n}{n})^{\frac{1}{4d+8}} $, with probability greater than $1-n^{-2}$, there are $a_i \in \{1,-1\}$ such that for all $0 \leq i <  m$, 
	\begin{align}
		\max_{x_j \in \mathcal{X}_n}|a_i v_{i,n,\epsilon}(j)-e_{i}(x_j)|\leq  \Omega_2 h^{1/2}.
	\end{align}
$\Omega_1$ depends on $d$, the  the diameter of $M$, $f_{\min}$, $\|f\|_{C^2}$, and $\Omega_2$ depends on $d$, the diameter and the volume of $M$, $f_{\min}$, and $\|f\|_{C^2}$. \cite{DUNSON2021282}
\end{theorem} 
\begin{remark}
In our setting, the sampling is uniform, thus $f \equiv 1/\text{Vol}(M)$,  which implies $\|f\|_{C^2} = 1/\text{Vol}(M)$. In addition, $d$, diameter, volume, injectivity radius, curvature and second order fundamental form are bounded uniformly, which are compact, thus $\mathcal{K}_1, \mathcal{K}_2, \mathcal{K}_3, \Omega_1,\Omega_2$ are bounded uniformly for all $M \in \MM$.
\end{remark}        
\begin{remark}
In the case when the eigenvalues are not simple, the same proof still works by introducing the eigenprojection.\cite{spectralapproximationoflinear}
\end{remark}
\begin{remark}
If we choose $h=(\frac{\log n}{n})^{\frac{1}{4d+13}}$, then for large enough $n$, we have
\begin{align}
	|\mu_{i,n,h}-\lambda_{i}|\leq \Omega_1 (\frac{\log n}{n})^{\frac{3}{8d+26}}, \quad \text{for}\ 0\leq i \leq m.
\end{align}
Similarly, if we choose $h=(\frac{\log n}{n})^{\frac{1}{4d+8}}$, we have
\begin{align}
	\max_{x_j \in \mathcal{X}_n}|a_i v_{i,n,\epsilon}(j)-e_{i}(x_j)|\leq  \Omega_2 (\frac{\log n}{n})^{\frac{1}{8d+16}}, \quad \text{for}\ 0\leq i \leq m.
\end{align}
Since $(\frac{\log n}{n})^{\frac{1}{4d+13}} \geq (\frac{\log n}{n})^{\frac{1}{4d+8}}$, the inequalities above hold at the same time if we choose $h=(\frac{\log n}{n})^{\frac{1}{4d+13}}$.
\end{remark}

Now, we introduce our algorithm and estimate the approximation error.
\begin{algorithm}[H]
	\renewcommand{\algorithmicrequire}{\textbf{Input: }}
	\renewcommand{\algorithmicensure}{\textbf{Output: }}
	\caption{Diffusion Map}
	\label{alg1}
	\begin{algorithmic}[1]
		\State \algorithmicrequire  point cloud $\mathcal{X}_n$, intrinsic dimension $d$, the lower bound of injectivity radius $\iota$ (\cref{injradius}), the lower bound of Ricci curvature $-\kappa$ (\cref{as:kappa}), upper bound of volume $V$ (\cref{ass:V}), norm error $\epsilon$ (\cref{range_of_e}).
		\State Calculate bindwidth $h=(\frac{\log n}{n})^{\frac{1}{4d+13}}$, diffusion time according to \cref{embeddingthm}, \cref{kernelerror}, \cref{injvol}
		\begin{align}
			t=\min\left\{t_0(d,\kappa,\iota,\epsilon),4,\frac{\iota^2}{4}\right\},
		\end{align}
		the heat kernel error according to \cref{sele eprime},
		\begin{align}
\epsilon^\prime=\frac{(4\pi t)^{-\frac{d}{2}}}{8}\exp\left(-\frac{\beta^2t}{4}-\frac{2\sqrt{3dt}\beta}{3}\right)
		\end{align}
		where $\beta =\sqrt{\kappa}(d-1)$
		and embedding dimension according to \cref{embeddingthm},  \cref{kernelerror}
		\begin{align}
			m=\max\left\{N_0(d,\kappa,\iota,\epsilon,V,\frac{t}{2}),N_1(d,\kappa,\iota,V,\epsilon^\prime,t)\right\}+1
		\end{align}
		\State Construct $W$ and $D$ with bandwidth $h$ according to \cref{Lap matrix}.
		\State Calculate first $m$ eigenvalues and eigenfunctions $\left\{\mu_{i,n,h},\tilde{v}_{i,n,h}\right\}_{i=1}^m$ of $L_n=\frac{D^{-1}W-I}{h^2}$.
		\State For $1\leq i \leq m$, calculate \begin{align}
			\mathbb{N}(i)=|B_h(x_i) \cap \left\{x_1,\dots,x_n\right\}|.
		\end{align}
		Calculate 
		\begin{align}
			\|\tilde{v}\|_{l^2(1/\hat{p})} = 	\sqrt{\frac{\omega(d-1)h^d}{d}\sum_{i=1}^{n}\frac{\tilde{v}^2(i)}{\mathbb{N}(i)}},
		\end{align}
		and normalize 
		\begin{align}
			v_{i,n,h}=\frac{\tilde{v}_{i,n,h}}{\|\tilde{v}_{i,n,h}\|_{l^2(1/\hat{p})}}.
		\end{align}
		\State Embed $x_j$ to $(t)^\frac{d+2}{4}\sqrt{2}(4\pi)^\frac{d}{4}\left(e^{-\mu_{i,n,h}t/2}v_{i,n,h}(j)\right)_{i=1}^m$
		\Ensure the embedding point clouds $\left\{(t)^\frac{d+2}{4}\sqrt{2}(4\pi)^\frac{d}{4}\left(e^{-\mu_{i,n,h}t/2}v_{i,n,h}(j)\right)_{i=1}^m\right\}_{i=1}^n \subset \R^m$.
	\end{algorithmic}
\end{algorithm}

We compare this embedding with embedding in \cref{sec:problem} and estimate the approximation error. 

To estimate the error term, we consider the error of $i$-th component, for large enough $n$, i.e. 
\begin{align}
	&(\frac{\log n}{n})^{\frac{1}{4d+13}} = h \leq \mathcal{K}_1 \min \left(\left(\frac{\min(\mathsf\Gamma_m,1)}{\mathcal{K}_2+\lambda_m^{d/2+5}}\right)^2,\, \frac{1}{(\mathcal{K}_3+\lambda_m^{(5d+7)/4})^2}\right)\\
	\iff& n > \frac{1}{\mathcal{K}} \max \left(\left(\frac{\mathcal{K}_2+\lambda_K^{d/2+5}}{\min(\mathsf\Gamma_K,1)}\right)^{8d+26},  (\mathcal{K}_3+\lambda_K^{(5d+7)/4})^{8d+26}\right),
\end{align}
up to a $\log$ factor.

Then we have 
\begin{align}
	|e^{-\mu_{i,n,h}t/2}v_{i,n,h}(j)-e^{-\lambda_it/2}e_i(x_j)| &\leq e^{-\mu_{i,n,h}t/2}|v_{i,n,h}(j)-e_i(x_j)|+|e^{-\mu_{i,n,h}t/2}-e^{-\lambda_it/2}|e_i(x_j) \\
&	\leq |v_{i,n,h}(j)-e_i(x_j)| + \frac{t}{2}|\mu_{i,n,h}-\lambda_i| \|e_i\|_{l^\infty}
\end{align}

We also have 
\begin{align}
	\|e_i\|_{l^\infty} \leq \|e_i\|_{L^\infty} \leq C \lambda_{i}^{\frac{n-1}{4}},
\end{align}
where $C$ only depends on the dimension $d$ of $M$, lower bound of injectivity radius and the absolute value of the sectional curvature\cite{ajm/1154098927}, which are compact set, thus $C$ has a uniformly upper bound for all $M\in \MM$. And we have proved $\lambda_{i}$ are bounded from above for $i\leq m$ with fixed $m$. Consequently, $\|e_i\|_{l^\infty}$ has the uniformly bound, we denote it as $C_\MM$.

Therefore,
\begin{align}
	|e^{-\mu_{i,n,h}t/2}v_{i,n,h}(j)-e^{-\lambda_it/2}e_i(x_j)| &\leq \Omega_2 (\frac{\log n}{n})^{\frac{1}{8d+16}} + \frac{t}{2} C_\MM \Omega_1 (\frac{\log n}{n})^{\frac{3}{8d+26}} \\
	&= \Omega (\frac{\log n}{n})^{\frac{1}{8d+16}},
\end{align}
hence
\begin{align}
|\left(e^{-\mu_{i,n,h}t/2}\tilde{v}_{i,n,h}(j)\right)_{i=1}^m - \left(e^{-\lambda_it/2}e_i(x_j)\right)_{i=1}^m| \leq \sqrt{m}\Omega (\frac{\log n}{n})^{\frac{1}{8d+16}}
\end{align}

\subsection{Related Result}

In this subsection, we discussed some other convergence results. In the following, we ignore the specific settings like sample size $n$, bandwidth $h$ and others, we only focus on the convergence rate and norm, when discussing the convergence, we also ignore the constant. We denote the eigenfunctions and eigenvalues of $\Delta$ as $e_i$, $\lambda_i$ and denote the eigenvectors and eigenvalues of graph laplacian as $\mu_{i}$, $v_{i}$, here we will not specify laplacian graph and  kernel function.

In \cite{garcia2020error}, we have similar result, the following holds for finite $k$
\begin{align}
	\frac{1}{n}\sum_{i=1}^n (v_k(i)-e_k(x_i))^2 = O\left((\frac{\log n}{n})^\frac{1}{2d}\right),\\
	|\mu_i - \lambda_{i}| =O\left((\frac{\log n}{n})^\frac{1}{2d}\right).
\end{align}

 And the author then improved the result to $O\left((\frac{\log n}{n})^\frac{1}{d+4}\right)$ under some other conditions in \cite{calder2022improved}. This convergence rate is much faster than we use, but the left side describe the average pointwise error. This result has also been improved\cite{calder2022lipschitz}, they obtained the convergence result for $l^\infty$, Lipschitz norm is $O\left((\frac{\log n}{n})^\frac{1}{d+4}\right)$ for specific manifold, however the constant $C$ in their result depends on $M$, which is difficult to quantify, so it cannot be applied to the family of manifolds.

In \cite{cheng2022eigen}, the author derived the convergence rate of eigenvalues is $O\left((\frac{\log n}{n})^\frac{1}{d/2+2}\right)$, and the convergence rate of eigenvectors is $O\left((\frac{\log n}{n})^\frac{1}{d/2+3}\right)$ under different settings, and it also includes the results we discussed above. In\cite{wang2015spectral}, the convergence rate is $O\left(n^{-\frac{2}{(5d+6)(d+6)}}\right)$ under $l^\infty$ norm, which is a slower than result we use.

\section{Bounding the finite sample error of the tangent space estimation}
\label{sec:tangent}

In the previous sections, we have shown that in the case of a well-behaved manifold $M$, the diffusion map embedding with finite, sufficiently large $m$ is still well behaved \mmp{be more explicit here} with respect to volume, smoothnes (Sobolev norm), 
pushforward density and reach.
We conclude the paper by applying these results to the tangent space estimation of $\varphi(M)$ from samples. 

We consider our model as the following. We have $n$ sample points $\mathcal{X}_n=\left\{X_{1},\dots,X_{n}\right\}$ which are sampled i.i.d. from $M$, then we perform diffusion map on it, we obtain $\mathbb{Y}_n=\left\{\hat{Y}_1,\dots,\hat{Y}_n\right\}$, which is the approximation of $\mathcal{Y}_n=\varphi(\mathcal{X}_n)=\left\{Y_{1},\dots,Y_{n}\right\}$ \wenyu{Maybe change notation here}, where $Y_i$ is the embedding via eigenvalues and eigenfunctions of $M$, and they are distributed identically and independently on $\varphi(M)$. And from \cref{sec:convergence}, the error $\sigma$ between $\mathbb{Y}_n$ and $\mathcal{Y}_n$ is $C\left(\frac{\log n}{n}\right)^{\frac{1}{8d+16}}$.

We will use a local polynomial estimator of degree $k$ to approximate the tangent space at point $X_j$.
Let $P_{n-1}(f) = \frac{1}{n-1}\sum_{i\neq 1} f(X_i-X_1)$, the integration with respect to the empirical distribution of the sample, excluding $X_1$. For a constant $t>0$ and a bandwidth $\tilde{h}>0$, the local polynomial estimator $\left(\hat{\Pi},\hat{T}_{2}\dots,\hat{T}_{k-1}\right)$ of the tangent space at $X_1$ is given by 
\begin{align}
	\mathop{\arg\min}\limits_{\Pi,\sup_{2\leq l \leq k}\|A_i\|_{\text{op}}\leq t}P_{n-1}^{(j)}\left[\left\|x-\Pi(x)-\sum_{l=2}^{k-1}A_l(\pi(x)^{\otimes l})\right\|^2 \mathbbm{1}_{B(0,\tilde{h})}(x) \right],
\end{align}
where $\Pi$ is an orthogonal projector on a $d$-dimensional subspace of $\R^m$, and $A_l,\,l=2,\ldots k-1$ are symmetric tensors of polynomial coefficients, of order $l$ from $(\R^m)^l$ to
$\R^m$.

Since $T_{X_1}M$ is the tangent space of $M$, thus it is best linear approximation of $M$ near $X_1$, $\hat{T}_1 := \text{Im} \hat{\Pi}_j$ is used to estimate it. By exchangeability, this holds for all other data points $X_2,\dots, X_n$. The distance between two subspaces $U$, $V$ of $\R^m$ is defined as 
\begin{align}
	\angle (U,V)= \|\Pi_U-\Pi_V\|.
\end{align}

Under standard conditions similar to ours, \cite{aamariestimate} derived the asymptotically optimal minimax error of this estimator. We reproduce their result here.
\begin{lemma}[\cite{aamariestimate}]
  \label{thm:aamari-tangent}
We denote $\mathcal{P}$ as the set of distributions $P$ over support $M \in \MM$ with sampling density $f$ such that $0< f_{\min} \leq f \leq f_{\max} < \infty$, and $\mathcal{P}(\sigma)$ is the set of distributions of r.v. $X=X_M+X_\perp$, where distribution of $X_M$ is in $\mathcal{P}$ and $X_\perp$ is perpendicular to $T_{X_M} M$, $|X_\perp|\leq \sigma$ and $\mathbb{E}(X_\perp|X_M)=0$.

Assume that $t \geq C_{k,d,\tau_{\min},\textbf{L}} \geq \sup_{2\leq i \leq k} \| T_i ^\star\|$. Set $\tilde{h}=\left(C_{d,k} \frac{f_{\max}^2 \log n}{f_{\min}^3 (n-1)}\right)^{\frac{1}{d}}$, for $C_{d,k}$ large enough, and assume that $\sigma\leq \tilde{h}/4$. If $n$ is large enough such that $\tilde{h}\leq h_0 = \frac{\tau_{\min} \wedge L^{-1}_\perp}{8}$, then with probability at least $1-(\frac{1}{n})^{k/d}$,
\begin{align}
	\max\limits_{1\leq j \leq n} \angle(T_{X_{M,j}}M,\hat{T}_j) \leq C_{d,k,\tau_{\min},\textbf{L}} \sqrt{\frac{f_{\max}}{f_{\min}}} (\tilde{h}^{k-1}\vee \sigma \tilde{h}^{-1})(1+t\tilde{h}).
\end{align}
Taking $t = \tilde{h}^{-1}$, for $n$ large enough,
\begin{align}
	\sup\limits_{P \in \mathcal{P}(\sigma)} \mathbb{E}_{P^{\otimes n}} \max\limits_{1\leq j \leq n} \angle(T_{X_{M,j}}M,\hat{T}_j) \leq C \left(\frac{\log n}{n-1}\right)^{\frac{k-1}{d}} \left\{1 \vee \sigma \left(\frac{\log n}{n-1}\right)^{-\frac{k}{d}}\right\}.
\end{align}
  \end{lemma}
Transferring \cref{thm:aamari-tangent} requires (1) controlling the $\textbf{L}^\prime$, pushforward density $g$ and reach $\tau_{\min}$ of $\varphi(M)$, achieved in \cref{sec:phim-L},\cref{sec:density},\cref{sec:reach} and (2) controlling the finite sample error in the estimation of the eigenfunctions $\varphi_{1:m}$, done in Section \ref{sec:convergence}.

Applying the results from \cref{thm:conver_rate} directly to \cref{thm:aamari-tangent}, we obtain
\begin{align}
	\sup\limits_{P \in \mathcal{P}(\sigma)} \mathbb{E}_{P^{\otimes n}} \max\limits_{1\leq j \leq n} \angle(T_{Y_{j}}\varphi(M),\hat{T}_j) &\leq C \left(\frac{\log n}{n-1}\right)^{\frac{k-1}{d}} \left\{1 \vee \left(\frac{\log n}{n}\right)^{\frac{1}{8d+16}} \left(\frac{\log n}{n-1}\right)^{-\frac{k}{d}}\right\} \\
	&\sim C \left(\frac{\log n}{n-1}\right)^{-\frac{1}{d}} \left(\frac{\log n}{n}\right)^{\frac{1}{8d+16}}\\
	&\sim C \left(\frac{\log n}{n}\right)^{-\frac{7d+16}{d(8d+16)}}.
\end{align}
 The rate of convergence for $\hat{T}_i$ is $O\left(\left(\frac{\log n}{n}\right)^{-\frac{7d+16}{d(8d+16)}}\right)$, which is not convergent. The reason is the decreasing rate of error is too slow.

We bypass this obstacle by using different sample sizes in the
Diffusion Maps calculation and tangent space estimation. Thus
$\hat{\varphi}$ is estimated on the full sample $\XX$, after which
$T\varphi(M)$ is estimated at a subset $\tilXX$ of the data points, with
$|\tilXX|=n^{\frac{1}{b}}=\tiln$, where $b>1$ is to be determined.
    The effect is a faster rate of convergence for $T_i$, due to the
    reduced error of the embedding $\hat{\varphi}$.

In this case, our error is small relative to the number of estimated sample points, i.e., error term is $O\left((\frac{\log n}{n})^{\frac{1}{8d+16}}\right)$, we denote its exponent as $\frac{1}{a}$, the error is relatively small enough if we select  \mmp{what is this? who is relatively small?}  $n^{\frac{d}{(8d+16)k}}$ sample points uniformly. This case is equivalent to sampling $n$ points with error term $O\left((\frac{\log n}{n})^{\frac{k}{d}}\right)$.

\subsection{Upper Bound}
Let $\hat{T}_i$ be a basis for the estimated tangent space $T_{\varphi(X_i)}\varphi(M)$. Now we estimate $\hat{T}_i$ at $\tiln=n^{\frac{1}{b}}$ sample
points $\tilde{\mathbb{Y}}$, while the entire sample $\XX$ with $n$ points is used
to estimate $\varphi$. We treat $n^\frac{1}{b}$ and $\lfloor n^\frac{1}{b}\rfloor$ as equivalent since this will not have effect on the rate of convergence.

\begin{theorem}[Diffusion Maps tangent space convergence upper bound]
	\label{thm:tangent-upper}
	Assume that $M\in \MM$ as before. The sample $\XX_n$ is mapped by $m$-dimensional Diffusion Maps to $\mathbb{Y}_n=\phihat(\mathcal{X}_n)$,  with $m,t$ fixed and the kernel width $h=(\frac{\log n}{n})^{\frac{1}{4d+13}}$. Then, on a uniformly sampled $\tilde{\mathbb{Y}}\subset \mathbb{Y}_n$ of size $\tiln=n^{\frac{d}{(8d+16)k}}$, the tangent space is estimated as $\That_i\in \rrr^{m\times d}$, for $\hat{Y}_i\in \tilde{\mathbb{Y}}$, with bandwidth $\tilde{h}=\left(C_{d,k} \frac{f_{\max}^2 \log \tiln}{f_{\min}^3 (\tiln-1)}\right)^{\frac{1}{d}}$. Then, 
		\begin{align}
		\sup_{P\in \mathcal{P}} \mathbb{E}_{P^{\otimes \tiln}} \max_{1\leq j \leq \tiln} \angle (T_{Y_{j}}\varphi(M),\hat{T}_j)\leq C \left(\frac{\log n }{n}\right)^\frac{k-1}{(8d+16)k}.
	\end{align}
\end{theorem}

\begin{proof}
	Under the conditions of the theorem, we have, for each embedding coordinate $j=1,\ldots m$
	\begin{align*}
		\sup_{P\in \mathcal{P}(\sigma_{\varphi})} \mathbb{E}_{P^{\otimes \tiln}} \max_{1\leq j \leq \tiln} \angle (T_{Y_j}\varphi(M),\hat{T}_i)\leq C\left(\frac{\log \tiln}{\tiln-1}\right)^\frac{k-1}{d}\left\{1\vee \sigma_{\varphi} \left(\frac{\log \tiln}{\tiln-1}\right)^{-\frac{k}{d}}\right\}
	\end{align*}
	where $\sigma_{\varphi}=O\left((\frac{1}{n})^{\frac{1}{a}}\right)$ is the error term computed using $n$ sample points.
	
	In the above, we have applied \cref{thm:aamari-tangent} here,
	disregarding the assumption that the noise $\varphi(X_i)-\phihat(X_i)$ is
	orthogonal to the manifold, zero mean and i.i.d.  Indeed, this
	assumption is not necessary for our theorem.  When we estimate upper
	bound $\sigma_{\varphi}$, we do not need the noise to be orthogonal and
	iid.  Lemma 2 and Lemma 3 of \cite{aamariestimate} are geometric,
	no need to assume iid or orthogonal. Furthermore, Proposition 2 is
	about $\left\{Y_1,\dots,Y_n\right\}$, we know they are iid. Hence, the
	proof of proof of Theorem \ref{thm:aamari-tangent} can apply to our
	case.

	\begin{align}
		\sigma_{\varphi} \left(\frac{\log \tiln}{\tiln-1}\right)^{-\frac{k}{d}}=O\left((\log n)^{\frac{1}{a}-\frac{k}{d}} (\frac{1}{n})^{\frac{1}{a}-\frac{k}{db}}\right)
	\end{align}
	if $b\geq \frac{ak}{d}$, $\sigma_{\varphi} \left(\frac{\log \tiln}{\tiln-1}\right)^{-\frac{k}{d}}<1$ when $n$ is large; if $b<\frac{ak}{d}$, $\sigma_{\varphi} \left(\frac{\log \tiln}{\tiln-1}\right)^{-\frac{k}{d}}>1$ when $n$ is large, so
	
	\begin{align}
		C\left(\frac{\log \tiln}{\tiln-1}\right)^\frac{k-1}{d}\left\{1\vee \sigma_{\varphi} \left(\frac{\log \tiln}{\tiln-1}\right)^{-\frac{k}{d}}\right\}=\left\{
		\begin{array}{rcl}
			O\left((\frac{\log n}{n})^{\frac{1}{a}-\frac{1}{db}}\right) & & {1<b<\frac{ak}{d}}\\
			O\left((\frac{\log n}{n})^{\frac{k-1}{db}}\right) & & {b\geq \frac{ak}{d}}
		\end{array} \right. .
	\end{align}
	
	To make it converge to $0$, $\frac{1}{a}-\frac{1}{db}>0$, which implies $b>\frac{8d+16}{d}$. We also want this bound as small as possible, we notice the power is increasing on $(1,\frac{ak}{d})$ and decreasing on $[\frac{ak}{d},\infty)$. Thus $b=\frac{ak}{d}=\frac{(8d+16)k}{d}$ minimize upper bound, and we have the upper bound is
	
	\begin{align}
		\sup_{P\in \mathcal{P}} \mathbb{E}_{P^{\otimes \tiln}} \max_{1\leq j \leq \tiln} \angle (T_{Y_{j}}\varphi(M),\hat{T}_j)\leq C \left(\frac{\log n }{n}\right)^\frac{k-1}{(8d+16)k}.
	\end{align}
\end{proof}

\begin{remark}
	\cite{aamari2018stability} also discussed the model without any assumption about noise except norm. We can use this alternative result, but this one leads to a  slower convergence rate \mmp{of XXXX} \cref{thm:aamari-tangent}
\end{remark}

\wenyu{This remark can be removed}
\wenyu{
\begin{remark}
	In \cite{aamariestimate}, the author assume the noise $X_\perp$ at $X_M$ are orthogonal to $T_{X_M} M$, $|X_\perp|\leq \sigma$ and $\mathbb{E}(X_\perp|X_M)=0$. However, in our setting over $\varphi(M)$, we only have the noise $|Y_j-Y_{\varphi(M),j}| \leq \sigma$. Actually, we only need $\left\{Y_{\varphi(M),1},\dots,Y_{\varphi(M),n}\right\}$ instead of $\left\{Y_1,\dots,Y_n\right\}$ are identically and independently distributed, it is easy to verify that $\mathcal{Y}_n$ is distributed identically and independently on the $M \in \mathcal{M}(d,\tau_{\min},\textbf{L}^\prime,k)$ since $\mathcal{X}_n$ are i.i.d. and $\mathcal{Y}_n$ are obtained via fixed eigenfunctions and eigenvalues and sampling density $f$. 
	
	The Lemma 2 and Lemma 3 in \cite{aamariestimate} are geometric result, we do not need i.i.d which is a probability assumption, and it is easy to check that orthogonal is not required.
	And Proposition 2 is about $\left\{Y_{\varphi(M),1},\dots,Y_{\varphi(M),n}\right\}$, where we need they are i.i.d, which aligns with our result.
	Proof of \cref{thm:aamari-tangent} uses Lemma and Proposition we mentioned above and is mainly estimation, i.i.d and orthogonal are still not required.
\end{remark}
}

\begin{remark}
	To obtain the points where we estimate tangent space, we can only use first $\tiln$ points after DM embedding $\left\{(t)^\frac{d+2}{4}\sqrt{2}(4\pi)^\frac{d}{4}\left(e^{-\mu_{i,n,h}t/2}v_{i,n,h}(j)\right)_{i=1}^m\right\}_{i=1}^{\tiln}$. When we obtain $\mu_{i,n,h}$ and $v_{i,n,h}$, we use $n$ points, so the error is still $O\left((\frac{\log n}{n})^{\frac{1}{8d+16}}\right)$. In addition, since $\mathcal{Y}_n$ are i.i.d, the first $\tiln$ embedding points are i.i.d from $\mathcal{Y}_{\tiln}$ with error, which satisfies our requirements. 
\end{remark}

\subsection{Remarks/Discussion}
\label{sec:tangent-discuss}

If we only use $\tiln = n^{\frac{d}{(8d+16)k}}$ sample points to estimate tangent space of $\varphi(\mathcal{M})$ and use all $n$ sample points to do Diffusion Map, we will have:
\begin{align}
	\sup_{P\in \mathcal{P}} \mathbb{E}_{P^{\otimes \tiln}} \max_{1\leq j \leq \tiln} \angle (T_{Y_{\varphi(M),j}}\varphi(M),\hat{T}_j)\leq C \left(\frac{\log n }{n}\right)^\frac{k-1}{(8d+16)k}.
\end{align}

When $n \rightarrow \infty$, $\tiln \rightarrow \infty$, all sample points come from the same distribution, so we think the selection is uniform, and the empirical distribution converge to true distribution uniformly a.s., so we can estimate embedding manifolds using random $ n^{\frac{d}{(8d+16)k}}$ points, thus this result makes sense.

We also need to find a balance between convergence rate and sample size for estimating tangent space. We know the convergence rate is
\begin{align}
	O\left((\frac{\log n}{n})^{\frac{db-(8d+16)}{db(8d+16)}}\right) \quad &\text{if} \quad  n^{\frac{d}{(8d+16)k}}<\tiln= n^{\frac{1}{b}} < n^{\frac{d}{8d+16}},\\
	O\left((\frac{\log n}{n})^{\frac{k-1}{db}}\right)\quad &\text{if} \quad  \tiln=n^\frac{1}{b}\leq n^{\frac{d}{(8d+16)k}}
\end{align}
which is increasing (order is larger) as $\tiln$ decreasing when $\tiln>n^{\frac{d}{(8d+16)k}}$, and decreasing as $\tiln$ decreasing when $\tiln<n^{\frac{d}{(8d+16)k}}$

Therefore, when $\tiln = n^{\frac{d}{(8d+16)k}}$, the convergence attain the maximum $O\left(\left(\frac{\log n }{n}\right)^\frac{k-1}{(8d+16)k}\right)$. But if we want to have more sample size, we can choose $\tiln<n^{\frac{d}{8d+16}}$, but in this case, the convergence rate is slower.

\wenyu{
If the manifolds are $\CC^{\infty}$, we can choose $k$  to be any positive integer, and we can choose an optimal $k$ to balance the ratio of tangent space and DiffusionMaps convergence rate.}

If we have more rapid convergence rate of Diffusion Maps error $\sigma_\varphi$, we can improve the rates in \mmp{Theorems WHICH?} \cref{thm:tangent-upper}, and if this rate is more rapid than $(\frac{1}{n})^\frac{1}{d}$ \wenyu{to have optimal rate, k/d is needed.}, we can use theorem from original \cite{aamariestimate} paper directly, that is, considering the tangent space of all points.

	\section{Conclusion}
\label{sec:conclusion}

	In this paper, we proved under some geometric assumptions and regularity conditions, manifolds family after DM still have good geometric properties. And with these properties and controlled error introduced by DM, we can estimate the tangent space at few points, i.e.
\begin{align}
	\sup_{P\in \mathcal{P}} \mathbb{E}_{P^{\otimes \tiln}} \max_{1\leq j \leq \tiln} \angle (T_{Y_{\varphi(M),j}}\varphi(M),\hat{T}_j)\leq C \left(\frac{\log n }{n}\right)^\frac{k-1}{(8d+16)k}.
\end{align} 

\section*{Acknowledgements}
M.M. gratefully acknowledges the DataShape Group at INRIA Saclay and the Institute for Mathematical and Statistical Innovation (IMSI) for hospitality while a portion this research was carried out.



\bibliographystyle{ieeetr}
\bibliography{refs}

@article{berard1994embedding,
	title={Embedding Riemannian manifolds by their heat kernel},
	author={B{\'e}rard, Pierre and Besson, G{\'e}rard and Gallot, Sylvain},
	journal={Geometric \& Functional Analysis GAFA},
	volume={4},
	pages={373--398},
	year={1994},
	publisher={Springer}
}

@article{aamariestimate,
	ISSN = {00905364, 21688966},
	URL = {https://www.jstor.org/stable/26581844},
	author = {Eddie Aamari and Clément Levrard},
	journal = {The Annals of Statistics},
	number = {1},
	pages = {pp. 177--204},
	publisher = {Institute of Mathematical Statistics},
	title = {NONASYMPTOTIC RATES FOR MANIFOLD, TANGENT SPACE AND CURVATURE ESTIMATION},
	urldate = {2024-11-22},
	volume = {47},
	year = {2019}
}

@article{portegies2016embeddings,
  title={Embeddings of Riemannian manifolds with heat kernels and eigenfunctions},
  author={Portegies, Jacobus W},
  journal={Communications on Pure and Applied Mathematics},
  volume={69},
  number={3},
  pages={478--518},
  year={2016},
  publisher={Wiley Online Library}
}

@article{wang2015spectral,
  title={Spectral convergence rate of graph Laplacian},
  author={Wang, Xu},
  journal={arXiv preprint arXiv:1510.08110},
  year={2015}
}

@article{von2008consistency,
  title={Consistency of spectral clustering},
  author={Von Luxburg, Ulrike and Belkin, Mikhail and Bousquet, Olivier},
  journal={The Annals of Statistics},
  pages={555--586},
  year={2008},
  publisher={JSTOR}
}

@article{cheng2022eigen,
  title={Eigen-convergence of Gaussian kernelized graph Laplacian by manifold heat interpolation},
  author={Cheng, Xiuyuan and Wu, Nan},
  journal={Applied and Computational Harmonic Analysis},
  volume={61},
  pages={132--190},
  year={2022},
  publisher={Elsevier}
}

@book{lee2003introduction,
  title={Introduction to Smooth Manifolds},
  author={Lee, J.M.},
  isbn={9780387954486},
  lccn={2002070454},
  series={Graduate Texts in Mathematics},
  url={https://books.google.com/books?id=eqfgZtjQceYC},
  year={2003},
  publisher={Springer}
}

@article{federer1959curvature,
  title={Curvature measures},
  author={Federer, Herbert},
  journal={Transactions of the American Mathematical Society},
  volume={93},
  number={3},
  pages={418--491},
  year={1959}
}

@article{aamari2019estimating,
author = {Eddie Aamari and Jisu Kim and Fr{\'e}d{\'e}ric Chazal and Bertrand Michel and Alessandro Rinaldo and Larry Wasserman},
title = {{Estimating the reach of a manifold}},
volume = {13},
journal = {Electronic Journal of Statistics},
number = {1},
publisher = {Institute of Mathematical Statistics and Bernoulli Society},
pages = {1359 -- 1399},
year = {2019},
doi = {10.1214/19-EJS1551},
URL = {https://doi.org/10.1214/19-EJS1551}
}

@book{federer1969geometric,
  title={Geometric Measure Theory},
  author={Federer, H.},
  isbn={9780387045054},
  lccn={69016846},
  series={Die Grundlehren der mathematischen Wissenschaften in Einzeldarstellungen},
  url={https://books.google.com/books?id=TALvAAAAMAAJ},
  year={1969},
  publisher={Springer}
}

@article{esteigenvalue,
author = {Li, Peter and Yau, Shing},
year = {1980},
month = {01},
pages = {},
title = {Estimates of eigenvalues of a compact Riemannian manifold},
volume = {36},
isbn = {9780821814390},
journal = {Proc. Sympos. PureMath.},
doi = {10.1090/pspum/036/573435}
}

@article{xu2005asymptotic,
  title={Asymptotic behavior of $L^2$-normalized eigenfunctions of the Laplace-Beltrami operator on a closed Riemannian manifold},
  author={Xu, Bin},
  journal={arXiv preprint math/0509061},
  year={2005}
}

@article{xu2009gradient,
author = {Shi, Yiqian and Xu, Bin},
year = {2009},
month = {05},
pages = {},
title = {Gradient estimate of an eigenfunction on a compact Riemannian manifold without boundary},
volume = {38},
journal = {Annals of Global Analysis and Geometry},
doi = {10.1007/s10455-010-9198-0}
}

@article{2004Xuhigherorder,
author = {Xu, Bin},
year = {2004},
month = {10},
pages = {231-252},
title = {Derivatives of the Spectral Function and Sobolev Norms of Eigenfunctions on a Closed Riemannian Manifold},
volume = {26},
journal = {Annals of Global Analysis and Geometry - ANN GLOB ANAL GEOM},
doi = {10.1023/B:AGAG.0000042902.46202.69}
}

@article{aamari2018stability,
  title={Stability and minimax optimality of tangential Delaunay complexes for manifold reconstruction},
  author={Aamari, Eddie and Levrard, Cl{\'e}ment},
  journal={Discrete \& Computational Geometry},
  volume={59},
  pages={923--971},
  year={2018},
  publisher={Springer}
}

@article{DUNSON2021282,
title = {Spectral convergence of graph Laplacian and heat kernel reconstruction in $L^\infty$ from random samples},
journal = {Applied and Computational Harmonic Analysis},
volume = {55},
pages = {282-336},
year = {2021},
issn = {1063-5203},
doi = {https://doi.org/10.1016/j.acha.2021.06.002},
url = {https://www.sciencedirect.com/science/article/pii/S1063520321000464},
author = {David B. Dunson and Hau-Tieng Wu and Nan Wu}
}

@article{calder2022lipschitz,
	title={Lipschitz regularity of graph Laplacians on random data clouds},
	author="Calder, Jeff and {Garcia Trillos}, Nicolas and Lewicka, Marta",
	journal={SIAM Journal on Mathematical Analysis},
	volume={54},
	number={1},
	pages={1169--1222},
	year={2022},
	publisher={SIAM}
}

@book{spectralapproximationoflinear,
author = {Chatelin, Françoise},
title = {Spectral Approximation of Linear Operators},
publisher = {Society for Industrial and Applied Mathematics},
year = {2011},
doi = {10.1137/1.9781611970678},
address = {},
edition   = {},
URL = {https://epubs.siam.org/doi/abs/10.1137/1.9781611970678},
eprint = {https://epubs.siam.org/doi/pdf/10.1137/1.9781611970678}
}

@article{ajm/1154098927,
author = {Harold Donnelly},
title = {{Eigenfunctions of the Laplacian on Compact Riemannian Manifolds}},
volume = {10},
journal = {Asian Journal of Mathematics},
number = {1},
publisher = {International Press of Boston},
pages = {115 -- 126},
year = {2006},
}

@book{do1992riemannian,
	title={Riemannian geometry},
	author={Do Carmo, Manfredo Perdigao and Flaherty Francis, J},
	volume={2},
	year={1992},
	publisher={Springer}
}

@book{lee2018introduction,
	title={Introduction to Riemannian manifolds},
	author={Lee, John M},
	volume={2},
	year={2018},
	publisher={Springer}
}

@article{haghverdi2015diffusion,
	title={Diffusion maps for high-dimensional single-cell analysis of differentiation data},
	author={Haghverdi, Laleh and Buettner, Florian and Theis, Fabian J},
	journal={Bioinformatics},
	volume={31},
	number={18},
	pages={2989--2998},
	year={2015},
	publisher={Oxford University Press}
}

@article{haghverdi2016diffusion,
	title={Diffusion pseudotime robustly reconstructs lineage branching},
	author={Haghverdi, Laleh and B{\"u}ttner, Maren and Wolf, F Alexander and Buettner, Florian and Theis, Fabian J},
	journal={Nature methods},
	volume={13},
	number={10},
	pages={845--848},
	year={2016},
	publisher={Nature Publishing Group}
}

@article{ferguson2011nonlinear,
	title={Nonlinear dimensionality reduction in molecular simulation: The diffusion map approach},
	author={Ferguson, Andrew L and Panagiotopoulos, Athanassios Z and Kevrekidis, Ioannis G and Debenedetti, Pablo G},
	journal={Chemical Physics Letters},
	volume={509},
	number={1-3},
	pages={1--11},
	year={2011},
	publisher={Elsevier}
}

@article{freeman2009photometric,
	title={Photometric redshift estimation using spectral connectivity analysis},
	author={Freeman, PE and Newman, JA and Lee, AB and Richards, JW and Schafer, CM},
	journal={Monthly Notices of the Royal Astronomical Society},
	volume={398},
	number={4},
	pages={2012--2021},
	year={2009},
	publisher={Blackwell Publishing Ltd Oxford, UK}
}

@article{boninsegna2015investigating,
	title={Investigating molecular kinetics by variationally optimized diffusion maps},
	author={Boninsegna, Lorenzo and Gobbo, Gianpaolo and No{\'e}, Frank and Clementi, Cecilia},
	journal={Journal of chemical theory and computation},
	volume={11},
	number={12},
	pages={5947--5960},
	year={2015},
	publisher={ACS Publications}
}

@article{yauupbound,
	author = {Li, Peter and Yau, Shing},
	year = {1986},
	month = {07},
	pages = {153-201},
	title = {On the parabolic kernel of the Schödinger operator},
	volume = {156},
	journal = {Acta Mathematica},
	doi = {10.1007/BF02399203}
}

@article{wang1997sharp,
	title={Sharp explicit lower bounds of heat kernels},
	author={Wang, Feng-Yu},
	journal={The Annals of Probability},
	volume={25},
	number={4},
	pages={1995--2006},
	year={1997},
	publisher={Institute of Mathematical Statistics}
}

@inproceedings{croke1980some,
	title={Some isoperimetric inequalities and eigenvalue estimates},
	author={Croke, Christopher B},
	booktitle={Annales scientifiques de l'{\'E}cole normale sup{\'e}rieure},
	volume={13},
	number={4},
	pages={419--435},
	year={1980}
}

@article{COIFMAN20065,
	title = {Diffusion maps},
	journal = {Applied and Computational Harmonic Analysis},
	volume = {21},
	number = {1},
	pages = {5-30},
	year = {2006},
	note = {Special Issue: Diffusion Maps and Wavelets},
	issn = {1063-5203},
	doi = {https://doi.org/10.1016/j.acha.2006.04.006},
	url = {https://www.sciencedirect.com/science/article/pii/S1063520306000546},
	author = {Ronald R. Coifman and Stéphane Lafon}
}

@article{belkin2008towards,
	title={Towards a theoretical foundation for Laplacian-based manifold methods},
	author={Belkin, Mikhail and Niyogi, Partha},
	journal={Journal of Computer and System Sciences},
	volume={74},
	number={8},
	pages={1289--1308},
	year={2008},
	publisher={Elsevier}
}

@article{gine2006empirical,
	title={Empirical graph Laplacian approximation of Laplace-Beltrami operators: large sample results},
	author={Gin{\'e}, Evarist and Koltchinskii, Vladimir},
	journal={Lecture Notes-Monograph Series},
	pages={238--259},
	year={2006},
	publisher={JSTOR}
}

@inproceedings{10.5555/3104322.3104459,
	author = {Ting, Daniel and Huang, Ling and Jordan, Michael I.},
	title = {An analysis of the convergence of graph Laplacians},
	year = {2010},
	isbn = {9781605589077},
	publisher = {Omnipress},
	address = {Madison, WI, USA},
	booktitle = {Proceedings of the 27th International Conference on International Conference on Machine Learning},
	pages = {1079–1086},
	numpages = {8},
	location = {Haifa, Israel},
	series = {ICML'10}
}

@article{belkin2003laplacian,
	title={Laplacian eigenmaps for dimensionality reduction and data representation},
	author={Belkin, Mikhail and Niyogi, Partha},
	journal={Neural computation},
	volume={15},
	number={6},
	pages={1373--1396},
	year={2003},
	publisher={MIT Press}
}

@article{hein2007graph,
	title={Graph laplacians and their convergence on random neighborhood graphs.},
	author={Hein, Matthias and Audibert, Jean-Yves and Luxburg, Ulrike von},
	journal={Journal of Machine Learning Research},
	volume={8},
	number={6},
	year={2007}
}

@article{jones2008manifold,
	title={Manifold parametrizations by eigenfunctions of the Laplacian and heat kernels},
	author={Jones, Peter W and Maggioni, Mauro and Schul, Raanan},
	journal={Proceedings of the National Academy of Sciences},
	volume={105},
	number={6},
	pages={1803--1808},
	year={2008},
	publisher={National Acad Sciences}
}

@article{garcia2020error,
	title={Error estimates for spectral convergence of the graph Laplacian on random geometric graphs toward the Laplace--Beltrami operator},
	author={Garc{\'\i}a Trillos, Nicol{\'a}s and Gerlach, Moritz and Hein, Matthias and Slep{\v{c}}ev, Dejan},
	journal={Foundations of Computational Mathematics},
	volume={20},
	number={4},
	pages={827--887},
	year={2020},
	publisher={Springer}
}

@article{belkin2006convergence,
	title={Convergence of Laplacian eigenmaps},
	author={Belkin, Mikhail and Niyogi, Partha},
	journal={Advances in neural information processing systems},
	volume={19},
	year={2006}
}

@article{calder2022improved,
	title={Improved spectral convergence rates for graph Laplacians on $\varepsilon$-graphs and k-NN graphs},
	author={Calder, Jeff and {Garcia Trillos}, Nicolas},
	journal={Applied and Computational Harmonic Analysis},
	volume={60},
	pages={123--175},
	year={2022},
	publisher={Elsevier}
}

@article{hassannezhad2016eigenvalue,
	title={Eigenvalue inequalities on Riemannian manifolds with a lower Ricci curvature bound},
	author={Hassannezhad, Asma and Kokarev, Gerasim and Polterovich, Iosif},
	journal={Journal of Spectral Theory},
	volume={6},
	number={4},
	pages={807--835},
	year={2016}
}

@article{bates2014embedding,
	title={The embedding dimension of Laplacian eigenfunction maps},
	author={Bates, Jonathan},
	journal={Applied and Computational Harmonic Analysis},
	volume={37},
	number={3},
	pages={516--530},
	year={2014},
	publisher={Elsevier}
}

@article{lafon2006diffusion,
	title={Diffusion maps and coarse-graining: A unified framework for dimensionality reduction, graph partitioning, and data set parameterization},
	author={Lafon, Stephane and Lee, Ann B},
	journal={IEEE transactions on pattern analysis and machine intelligence},
	volume={28},
	number={9},
	pages={1393--1403},
	year={2006},
	publisher={IEEE}
}

@article{coifman2008diffusion,
	title={Diffusion maps, reduction coordinates, and low dimensional representation of stochastic systems},
	author={Coifman, Ronald R and Kevrekidis, Ioannis G and Lafon, St{\'e}phane and Maggioni, Mauro and Nadler, Boaz},
	journal={Multiscale Modeling \& Simulation},
	volume={7},
	number={2},
	pages={842--864},
	year={2008},
	publisher={SIAM}
}

@article{arvizu2016dimensionality,
	title={Dimensionality reduction in transient simulations: A diffusion maps approach},
	author={Arvizu, Claudia MC and Messina, Arturo R},
	journal={IEEE Transactions on Power Delivery},
	volume={31},
	number={5},
	pages={2379--2389},
	year={2016},
	publisher={IEEE}
}

@article{Cheng1981,
	author = {Cheng, Shiu-Yuen and Li, Peter},
	journal = {Commentarii mathematici Helvetici},
	pages = {327-338},
	title = {Heat kernel estimates and lower bound of eigenvalues.},
	url = {http://eudml.org/doc/139873},
	volume = {56},
	year = {1981},
}


\appendix
\section{Appendix / supplemental material}

\renewcommand{\thetheorem}{A.\arabic{theorem}}

\subsection{Pushforward Density}

Here we list the lemma and theorem for estimating pushforward density.
\setcounter{theorem}{0}
\begin{lemma}{(Area Formula)}
	\label{lem:area}
	If $f: \mathbb{R}^m \rightarrow \mathbb{R}^n$ is Lipschitzian and $m \leq n$, then
	\begin{align*}
		\int_A g(f(x)) J_f(x) d \mathcal{L}^m x = \int_{\mathbb{R}^n} g(y) N(f|A,y) d\mathcal{H}^m y
	\end{align*}
	where $A$ is an lebesgue measurable set, $J_f(x)$ is the Jacobian $\sqrt{\det(d_x f^T d_x f)}$ and $g: \mathbb{R}^n \rightarrow \mathbb{R}$ and $N(f|A,y) < \infty$ for $\mathcal{H}^m$ almost all $y$.\cite{federer1969geometric}
\end{lemma}

\begin{theorem}
	\label{thm:push_density}
	If $P$ is distribution on $\mathcal{M}$ with density $f$ with respect to the $d$-dimensional Hausdorff measure and $\varphi$ is a diffeomorphism, then the density $g(p^\prime)$ with respect to $d$-dimensional Hausdorff measure of $P^\prime:= \varphi_\# P$ is 
	\begin{align*}
		g(p^\prime)=f(p)/\sqrt{\det \left(\pi_{T_p M} \circ d\varphi_p^T \circ d\varphi_p |_{T_pM}\right)}
	\end{align*}
	where $p=\varphi^{-1}(p^\prime)$. 
\end{theorem}

\begin{proof}
	Let $p \in \mathcal{M}$ be a fixed point and choose $r$ small enough such that exponential map $\Psi_p: T_p \mathcal{M} \rightarrow \mathbb{R}^D$ is an injection onto $B(p,r)\cap \mathcal{M}$. 
	
	Choose $A \subset B(p,r)\cap \mathcal{M}$, by definition of pushforward measure, we have
	\begin{align*}
		\int_{\varphi(A)} d P^\prime=\int_A dP =\int_A f(y) d\mathcal{H}^d y
	\end{align*}
	Since $\Psi_p$ is an injection, so $N(\Psi_p|\Psi_p^{-1}(A),y)=1$ for all $y \in A$ and $0$ otherwise. By \cref{lem:area}, we have
	\begin{align*}
		\int_A f(y) d\mathcal{H}^d y=\int_{\Psi_p^{-1}(A)} f(\Psi_p(x)) J_{\Psi_p}(x) d \mathcal{L}^d x
	\end{align*}
	And $\varphi$ is diffeomorphism, so $h=\varphi \circ \Psi_p$ is also injective, so
	\begin{align*}
		\int_{\Psi_p^{-1}(A)} f(\Psi_p(x)) J_{\Psi_p}(x) d \mathcal{L}^d x&=\int_{\Psi_p^{-1}(A)} f(\varphi^{-1}(h(x))) \frac{J_{\Psi_p}(h^{-1}(h(x)))}{J_{\varphi \circ \Psi_p}(h^{-1}(h(x)))} J_{\varphi \circ \Psi_p}(x)) d \mathcal{L}^d x \\
		&=\int_{\varphi(A)} f(\varphi^{-1}(z)) \frac{J_{\Psi_p}(h^{-1}(z))}{J_{\varphi \circ \Psi_p}(h^{-1}(z))} d \mathcal{H}^dz
	\end{align*}
	Thus 
	\begin{align*}
		\int_{\varphi(A)} d P^\prime=\int_{\varphi(A)} f(\varphi^{-1}(z)) \frac{J_{\Psi_p}(h^{-1}(z))}{J_{\varphi \circ \Psi_p}(h^{-1}(z))} d \mathcal{H}^dz
	\end{align*}
	$T_p \mathcal{M}$ is a subspace of $\mathbb{R}^D$ with dimension $d$, so we can choose basis of $T_p \mathcal{M}$ such that all elements in diagonal of transformation matrix are $1$. We notice that $h^{-1}(z)=\Psi^{-1}_{\varphi^{-1} (z)} \circ \varphi^{-1} (z)=0$, and 
	$d_0 \Psi_p = I_D + d_0 N_p= I_D$ is inclusion map. Thus
	\begin{align*}
		f(\varphi^{-1}(z)) \frac{J_{\Psi_p}(h^{-1}(z))}{J_{\varphi \circ \Psi_p}(h^{-1}(z))}=f(p)/\sqrt{\det \left(\pi_{T_p M} \circ d\varphi_p^T \circ d\varphi_p |_{T_pM}\right)}
	\end{align*}
	where $p=\varphi^{-1}(z)$.
\end{proof}

\subsection{Heat Kernel Estimate}
\begin{theorem}[Upper Bound for The Heat Kernel]
	\label{thm:upper_app}
	Let $M$ be a complete manifold without boundary. If the Ricci curvature of M is bounded from below by $-\kappa$ for some constant $\kappa\geq 0$, then for any $1<\alpha_1<2$ and $0<\alpha_2<1$, the heat kernel satisfies:
	\begin{align}
		K(t,p,q) \leq C(\alpha_2)^{\alpha_1} V^{-1/2}(B_p(\sqrt{t}))V^{-1/2}(B_q(\sqrt{t})) \exp \left(C(d)\alpha_2(\alpha_1-1)^{-1}\kappa t-\frac{d(p,q)^2}{(4+\alpha_2)t}\right),
	\end{align}
	where $B_x(r)$ is geodesic ball centered at $p$ with radius $r$, $C(\alpha_2)$ depends on $\alpha_2$ with $C(\alpha_2) \rightarrow \infty$ as $\alpha_2\rightarrow 0$.\cite{yauupbound}
\end{theorem}

\begin{lemma}[Estimate for Volume of Geodesic Ball]
	\label{injvol}
	Let $M$ be a complete manifold without boundary. then for $r \leq \iota_M /2$, we have 
	\begin{align}
		Vol(B_p(r)) \geq C^\prime(d) r^d,
	\end{align}
	where $C^\prime(d) = \frac{2^{d-1} \omega(d-1)^d}{d^d \omega (d)^{d-1}}$, and
	$\omega(d)=\frac{2\pi^{\frac{d}{2}}}{\Gamma(\frac{d}{2})}$ is the volume of the unit $d$-dimensional sphere \cite{croke1980some}. Thus $C^\prime(d)= \frac{2^d \Gamma(\frac{d}{2})^{d-1}}{d^d \Gamma(\frac{d-1}{2})^d}$.
\end{lemma}

\begin{theorem}[Lower Bound for Heat Kernel]
	\label{thm:lower_app}
	Let M be a complete Riemannian manifold of dimension $d$ with Ricci curvature is greater than $-\kappa(d-1)$ for some $\kappa\geq 0$. For any $t, \sigma>0$ and $p,q\in M$, we have
	\begin{align}
		K(t,p,q) \geq (4\pi t)^{-d/2} \exp \left(-\left(\frac{1}{4t}+\frac{\sigma}{3\sqrt{2t}}\right)d(p,q)^2-\frac{\beta^2 }{4}t-(\frac{\beta^2}{4\sigma}+\frac{2d\sigma}{3})\sqrt{2t}\right),
	\end{align}
	where $\beta=\sqrt{\kappa}(d-1)$.\cite{wang1997sharp}
\end{theorem}

\subsection{Global Reach}
\begin{lemma}
	\label{app:global_case}
	If $d(p,q)\leq s_0$ where $s_0=2\sqrt{2}\tau_{l}$, then $p, q$ cannot satisfy global reach case.
\end{lemma}
\begin{proof}
	If not, we assume there exists $p,q \in M$ such that $d(p,q)\leq s_0$ and $p,q$ satisfy global reach case, i.e. 
	\begin{align}
		
&\frac{p+q}{2} \in Med(M), \\
 \pi(\frac{p+q}{2})=p,q &\quad \text{and} \quad d_E(p,M)=d_E(q,M)=|p-q|/2 \label{proj_med}
	\end{align}
Since $d(p,q)\leq s_0$, we have $d(p,q)\leq \frac{3}{2}|p,q|$.

	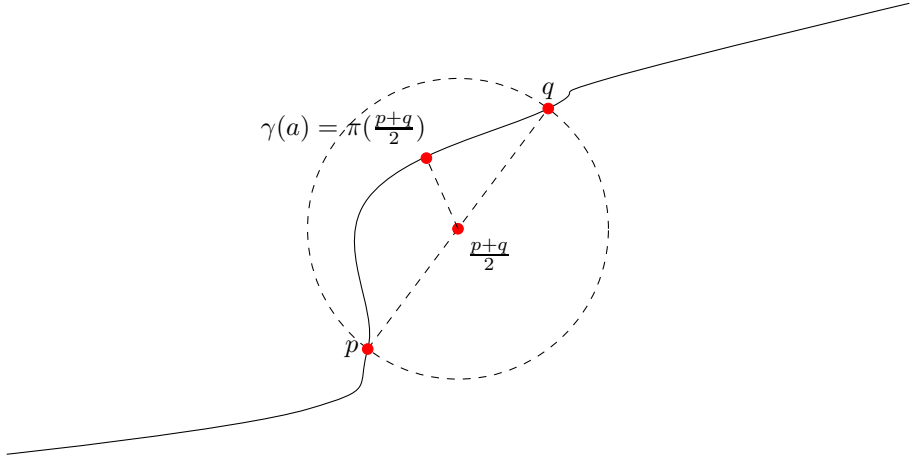
\begin{figure}[H]
		\centering
		\begin{tikzpicture}[scale=2]
			
			\draw[dashed] (0,0) circle(1cm);
			
			\draw[dashed] (-0.6,-0.8) -- (0.6,0.8);
			
			\draw plot[smooth, tension=0.7] coordinates {(-3,-1.5) (-1,-1.2) (-0.6,-0.8) (-0.6,0.2) (0.6,0.8) (1,1) (3,1.5)};
			\node at (-0.6,-0.8) [left] {$p$};
			\node at (0.6,0.8) [above] {$q$};
			\node at (0,0) [below right] {$\frac{p+q}{2}$};
			\node at (-0.15,0.5) [above left] {$\gamma(a)=\pi(\frac{p+q}{2})$};
			\filldraw[black] (-0.6,-0.8) circle (1pt);
			\filldraw[black] (0.6,0.8) circle (1pt);
			\filldraw[black] (0,0) circle (1pt);
			\filldraw[black] (-0.21,0.47) circle (1pt);
			\draw[dashed] (0, 0) -- (-0.21,0.47);
		\end{tikzpicture}
		\caption{Local Estimate}
	\end{figure}
	
	The inequality above implies the geodesic $\gamma(s)$ connecting $p$ and
	$q$ must be partly inside the ball
	$B(\frac{p+q}{2},\frac{|p-q|
	}{2})$, since $\overline{pq}$ is the diameter of this ball and
	\begin{align}
		|\gamma(s)|=d(p,q) \leq \frac{3}{2} |p-q| < \frac{\pi}{2}|p-q|,
	\end{align}
	which implies there exist $a \in (0,s)$ such that
	\begin{align}
		|\gamma(a)-\frac{p+q}{2}| < \frac{|p-q|}{2},
	\end{align}
	which contradicts with $\varphi(p),\varphi(q)$ satisfy global reach case \cref{proj_med}.
\end{proof}

\subsection{Selection of $t_0$, $m$}

\subsubsection{Selection of $t_0$}
For $t_0$, see theorem 4.4 in \cite{portegies2016embeddings}.

We define 
\begin{align}
	\Gamma(s,x,y) := r^d K(sr^2, u^{-1}(xr),u^{-1}(yr)),
\end{align}
where $u : B_r(p) \rightarrow \mathbb{R}^d$ is the harmonic coordinates.

We select $R_1 = R_1(d,\epsilon)$ such that, $\frac{1}{2}\leq s \leq 2$, 
\begin{align}
	\int_{\mathbb{R}^d \backslash B_{R_1}(0)} |\nabla \Gamma(s,0,y)|^2 dy < \epsilon.
\end{align}

We set $\alpha=\frac{1}{2}$ and $Q>1$ such that 
\begin{align}
	2(Q-1)C(d,\alpha) \leq \sigma,
\end{align}
where $C(d,\alpha)$ is the constant in Lemma 6.1 of \cite{portegies2016embeddings} and $\sigma \leq \sigma_1(d,\epsilon)$ such that $C(d,\alpha)\sigma |B_{R_1}(0)| < \epsilon$.

We also set $R_0=R(d,\alpha, C(d),Q)$ is the radius in Lemma 6.1, where $C(d)$ is the constant as below:
\begin{align}
	|\nabla\Gamma(s,x,y)| \leq \frac{C(d)}{s^{(d+1)/2}}\exp\left(-\frac{|x-y|^2}{8}\right).
\end{align}

Let $r_h := r_h(d,\kappa,\iota,\alpha,Q)$ be the harmonic radius, and set $r_3 = r_3(d,\kappa,\iota,\epsilon)< r_h/2$ such that for $t< 2r_3^2$, then
\begin{align}
	(2t)^{\frac{d+2}{2}}\int_{M\backslash B_{r_h/2}(p)}|\nabla K(t,p,q)|^2 dq <\epsilon.
\end{align}

Now we set 
\begin{align}
	r_0 = \min \left\{\frac{r_h}{R_0},\frac{r_h}{R_1},r_3\right\},
\end{align}
then $t_0 = r_0^2/2$.

\subsubsection{Selection of $N_0$}
According to theorem 4.4 in \cite{portegies2016embeddings}, 
\begin{align}
	\mathcal{H}(p)(q):= (2t)^{\frac{d+2}{4}}\sqrt{2}(4\pi)^{d/4} K(t,p,q)
\end{align}
is an embedding of $M$ into $L^2(M)$ with 
\begin{align}
	1-\epsilon < \|(d\mathcal{H})_p\|<1+\epsilon.
\end{align}

In addition, there is an isometry from $L^2(M)$ to $l^2$ by
\begin{align}
	U(f)_j=\int_M f(q) e_j(q) dq.
\end{align}

Then,
\begin{align}
	p\in M \overset{\mathcal{H}}{\rightarrow} \mathcal{H}(p) \in L^2(M) \overset{U}{\rightarrow} U(\mathcal{H}(p))
\end{align}
is an embedding, i.e.
\begin{align}
	\mathcal{F}(p) := (2t)^{\frac{d+2}{4}}\sqrt{2}(4\pi)^{d/4} \left(e^{-\lambda_1 t} e_1(p),e^{-\lambda_2 t} e_2(p),\dots\right)
\end{align}
is an embedding. We ignore the $e_0$ since it is the constant.

Therefore, for any $v \in T_p M$ with $|v|=1$, we have
\begin{align}
	2 (2t)^\frac{d+2}{2}(4\pi)^{d/2} \sum_{i=1}^{\infty} e^{-2\lambda_i t} (\nabla e_i (v))^2 = \|(d\mathcal{F})_p\|^2= \|(d\mathcal{H})_p\|^2 \in \left((1-\epsilon)^2,(1+\epsilon)^2\right).
\end{align}

In \cref{sec:phim-higher-deri}, we derived 
\begin{align}
	\|\nabla e_j\| \leq C \lambda_j^{\frac{d+2}{4}},
\end{align}
thus 
\begin{align}
2 (2t)^\frac{d+2}{2}(4\pi)^{d/2} \sum_{i=1}^{\infty} e^{-2\lambda_i t} (\nabla e_i (v))^2 \leq 2C (2t)^\frac{d+2}{2}(4\pi)^{d/2} \sum_{i=1}^{\infty} e^{-2\lambda_i t}  \lambda_j^{\frac{d+2}{2}}.
\end{align}

Furthermore, $\lambda_i$ is bounded from below by
\begin{align}
	\lambda_i ^ {d/2} \geq \alpha(d) \frac{i}{V} \left(\frac{V}{\int_{0}^{\text{diam}} F_\kappa(r)dr}\right)^{n(n+1)/(n-1)},
\end{align}
where $F_\kappa(r) = \begin{cases} 
	(-\kappa)^{-1/2} (\sinh \sqrt{-\kappa} r)^{n-1} & \text{if } \kappa<0, \\
	r^{n-1} & \text{if } \kappa = 0, \\
	(\kappa)^{-1/2} (\sinh \sqrt{\kappa} r)^{n-1} & \text{if } \kappa<0.
\end{cases}$ and $\kappa(n-1)$ is the lower bound of Ricci curvature of $M$\cite{Cheng1981}. Therefore, $\lambda_i \geq c i^{2/d}$ for some uniform constant $c$, when $i$ is large enough, we have 
\begin{align}
	e^{-2\lambda_it} \lambda_i^ {\frac{d+2}{2}} \leq e^{-\lambda_i t} \leq e^{-c i^{2/d} t},
\end{align}
then
\begin{align}
2C (2t)^\frac{d+2}{2}(4\pi)^{d/2} \sum_{i=N_0}^{\infty} e^{-2\lambda_i t}  \lambda_j^{\frac{d+2}{2}} \leq C^\prime \sum_{i=N_0}^{\infty} e^{-c i^{2/d} t}.
\end{align}

It is easy to verify that the tail can be controlled to arbitrary small, thus we can truncate $\mathcal{F}$ by $\mathcal{F}^{N_0}$:
\begin{align}
	\mathcal{F}^{N_0}(p) := (2t)^{\frac{d+2}{4}}\sqrt{2}(4\pi)^{d/4} \left(e^{-\lambda_1 t} e_1(p),\dots, e^{-\lambda_{N_1} t} e_{N_1}(p)\right)
\end{align}
such that $1-\epsilon < \|\mathcal{F}^{N_0}\| <1+\epsilon$.

\subsubsection{Selection of $N_1$}

We need find $N_1=N_1(d,\kappa,\iota,V,\epsilon^\prime,t_0)$ such that $|K_m(t^\prime,p,\cdot)-K(t^\prime,p,\cdot)|\leq \epsilon^\prime$ holds for any $m\geq N_1$ and any $t^\prime$ with $t_0\leq t^\prime \leq 4$.

In fact,
\begin{align}
	|K_{m}(t^\prime,p,q)-K(t^\prime,p,q)| = |\sum_{i=m}^{\infty} e^{-\lambda_i t^\prime} e_i(p)e_i(q)| \\
	\leq \sum_{i=N_1}^{\infty} e^{-\lambda_i t^\prime}| e_i(p)| |e_i(q)|.
\end{align}
We also have the estimate for L-infinity norm of eigenfunction \cite{ajm/1154098927}:
\begin{align}
	\|e_i(p)\|_{L^\infty} \leq C(d,\kappa,\iota) \lambda_i^ {\frac{d-1}{4}} \|e_i(p)\|_{L^2}.
\end{align}
Therefore, we have
\begin{align}
		|K_m(t^\prime,p,q)-K(t^\prime,p,q)| \leq C(d,\kappa,\iota) \sum_{i=N_1}^{\infty} e^{-\lambda_i t^\prime} \lambda_i^ {\frac{d-1}{2}}.
\end{align}

Similarly, $\lambda_i \geq 2c i^{2/d}$ for some uniform constant $c$, when $i$ is large enough, we have 
\begin{align}
	e^{-\lambda_i t^\prime} \lambda_i^ {\frac{d-1}{2}} \leq e^{-\lambda_i t^\prime /2} \leq e^{-c i^{2/d} t^\prime},
\end{align}
thus
\begin{align}
|K_m(t^\prime,p,q)-K(t^\prime,p,q)| \leq C(d,\kappa,\iota) \sum_{i=N_1}^{\infty} e^{-c i^{2/d} t^\prime}.
\end{align}

It is obvious that $\sum_{i=1}^{\infty} e^{-c i^{2/d} t^\prime}$ converges, so there exists  $N_1=N_1(d,\kappa,\iota,V,\epsilon^\prime,t_0)$ such that $|K_m(t^\prime,p,\cdot)-K(t^\prime,p,\cdot)|\leq \epsilon^\prime$. 

We notice that more precise estimate of eigenvalues and eigenfunction can make $N_1$ smaller, here we only list one possible method for selection.

\subsection{Example of $S^2$}
In this section, we consider the special case of $S^2$ with radius $1$ to verify the \cref{star}.

For $S^2$, the intrinsic dimension $d=2$, the injectivity radius is $\pi$, the sectional curvature and the Ricci curvature are both $1$, thus negative bound $\kappa=0$, $\beta=\sqrt{\kappa}(d-1)=0$ . 

The eigenvalues of $S^2$ is $l(l+1)$, with multiplicity $2l+1$, and the complex-valued eigenfunctions are 
\begin{align}
	Y_l^m(\theta,\phi)=N_l^m P_l^m(\cos \theta) e^{im\phi}\quad l\in \N, m=0,\pm1,\dots,\pm l,
\end{align} 
where $N_l^m$ is the normalization factor, $P_l^m$ is the associated Legendre polynomials, then the real valued eigenfunctions are
\begin{align}
Y_{lm} &= \begin{cases} 
	\sqrt{2}(-1)^m\Re(Y_l^{m}) & \text{if } m>0, \\
	Y_l^0 & \text{if } m = 0, \\
	\sqrt{2}(-1)^m\Im(Y_l^{-m}) & \text{if } m<0.
\end{cases} \\
&=\begin{cases} 
	(-1)^m \sqrt{2}\sqrt{\frac{2l+1}{2\pi} \frac{(l-m)!}{(l+m)!}}P_l^m(\cos \theta)\cos(m\phi) & \text{if } m>0, \\
	\sqrt{\frac{2l+1}{4\pi}} P_l^m(\cos \theta) & \text{if } m = 0, \\
		(-1)^m \sqrt{2}\sqrt{\frac{2l+1}{2\pi} \frac{(l-|m|)!}{(l+|m|)!}}P_l^{|m|}(\cos \theta)\sin(|m|\phi) & \text{if } m<0.
\end{cases}
\end{align}

We first consider embedding $S^2$ into $L^2(S^2)$ as the following:
\begin{align}
	f : p \rightarrow (4t) \sqrt{2\pi} K(t,p, \cdot).
\end{align}
Its operator norm is
\begin{align}
\|df_p\|^2= \sup_{\|v\|=1} 32\pi t^2 \int_{S^2} |\nabla_v K(t,p,q)|^2 dq.
\end{align}

We need find $t_0$ such that for any $0<t<t_0$, 
\begin{align}
	(1-\epsilon)^2<\sup_{\|v\|=1} 32\pi t^2 \int_{S^2} |\nabla_v K(t,p,q)|^2 dq < (1+\epsilon)^2.
\end{align}

The coordinate represent of $S^2$ is 
\begin{align}
	r(\theta, \phi) = (\sin \theta \cos \phi, \sin \theta \sin \phi, \cos \theta),\quad  \theta \in (0,\pi), \phi \in (0,2\pi),
\end{align}
and thus we can select unnormalized coordinate basis as the following
\begin{align}
	\frac{\partial}{\partial \theta} &= (\cos \theta \cos \phi, \cos \theta \sin \phi, -\sin \theta), \\
	\frac{\partial}{\partial \phi}  &= (-\sin\theta \sin \phi, \sin \theta\cos \phi,0).
\end{align}
Therefore, $g_{\theta \theta}=1 , g_{\theta \phi}=0, g_{\phi \phi}=\sin^2 \theta$. We normalize basis and then we have the gradient of $f$
\begin{align}
	\nabla f = \frac{\partial f}{\partial \theta} e_\theta + \frac{\partial f}{\partial \phi} \frac{1}{\sin \theta} e_\phi,
\end{align}
where $e_\theta = \frac{\partial}{\partial \theta}$, $e_\phi = \frac{1}{\sin \theta}\frac{\partial}{\partial \phi}$.

Then for $v= v^1 e_\theta + v^2 e_\phi$, we have
\begin{align}
	\nabla_v f = v(f)= df(v)=\left\langle \nabla f, v\right\rangle = \frac{\partial f}{\partial \theta} v^1 + \frac{\partial f}{\partial \phi} \frac{v^2}{\sin \theta}.
\end{align}

Now we compute $\int_{S^2} |\nabla_v K(t,p,q)|^2 dq$, without loss of generality, we can assume $v=\frac{\sqrt{2}}{2}e_\theta+\frac{\sqrt{2}}{2}e_\phi$ and $p=(\pi/2,0)$ ($\pi/2$ can simplify Riemannian metric.) due to symmetry of $S^2$.

We first prove $\int_{S^2} \frac{\partial K}{\partial \theta} \frac{\partial K}{\partial \phi} dq =0$. We know 
\begin{align}
	K(t,p,q) = \sum_{l=0}^\infty \sum_{m=-l}^{l} e^{-l(l+1) t} Y_{l}^m(p),\bar{Y}_{l}^m(q) \\
	=\sum_{l=0}^{\infty} \frac{2l+1}{4\pi}e^{-l(l+1) t} P_l(\cos d),
\end{align}
where the equality comes from spherical harmonic addition theorem, $d$ is the geodesic distance between $p$ and $q$ and $\cos d = \cos \theta_p \cos \theta_q + \sin \theta_p \sin \theta_q \cos(\phi_p -\phi_q)$.

Thus
\begin{align}
	\frac{\partial K}{\partial \theta}& = \sum_{l=0}^\infty c_l P_l^\prime(\cos d) \left(-\sin \theta_p \cos \theta_q + \cos \theta_p \sin \theta_q \cos (\phi_p -\phi_q) \right), \\
	\frac{\partial K}{\partial \phi}& = \sum_{l=0}^\infty c_l P_l^\prime(\cos d) \left(-\sin \theta_p \sin \theta_q \sin (\phi_p -\phi_q) \right).
\end{align}
At $p = (\pi/2 ,0 )$, they are
\begin{align}
\frac{\partial K}{\partial \theta} &= \sum_{l=0}^\infty c_l P_l^\prime(\sin \theta_q \cos \phi_q) \left( -\cos \theta_q \right), \\
\frac{\partial K}{\partial \phi} &= \sum_{l=0}^\infty c_l P_l^\prime(\sin \theta_q \cos \phi_q) \left( \sin \theta_q \sin \phi_q \right).
\end{align}

We only need to check 
\begin{align}
	\int_{S^2} P^\prime_{l}(\sin \theta_q \cos \phi_q) P^\prime_{l^\prime}(\sin \theta_q \cos \phi_q)  \cos \theta_q \sin \theta_q \sin \phi_q dq= 0.
\end{align}

We notice that $P_l$ is the Legendre polynomial, so $P_l^\prime$ is a polynomial. All terms are looks like $(\sin \theta_q \cos \phi_q)^k$.

It is easy to verify
\begin{align}
&\int_{S^2} (\sin \theta_q \cos \phi_q)^k  \cos \theta_q \sin \theta_q \sin \phi_q dq\\
=&\int_{0}^{2 \pi} \int_0^\pi \sin^{k+2}(\theta_q) \cos \theta_q \cos^k (\phi_q) \sin \phi_q d\phi d\theta  \\
=& \int_{0}^{2 \pi}\cos^k (\phi_q) \sin \phi_q d\phi \int_0^\pi \sin^{k+2}(\theta_q) \cos \theta_q d\theta_q =0.
\end{align}
Therefore, add them up and we get the result we need. And immediately, we have
\begin{align}
\int_{S^2} |\nabla_v K(t,p,q)|^2 dq &= \int_{S^2} \left|\frac{\sqrt{2}}{2} \frac{\partial K }{\partial \theta}+ \frac{\sqrt{2}}{2} \frac{\partial K }{\partial \phi}\right|^2 dq\\
&= \frac{1}{2} \int_{S^2} \left(\frac{\partial K }{\partial \theta}\right)^2 + \left(\frac{\partial K }{\partial \phi}\right)^2 + 2\frac{\partial K }{\partial \theta}\frac{\partial K }{\partial \phi} dq\\
&=\frac{1}{2}\int_{S^2} \left\langle \nabla K, \nabla K\right\rangle dq
\end{align}

Then by Green Identity,
\begin{align}
\int_{S^2} \left\langle \nabla K, \nabla K\right\rangle dq = -\int_{S^2} K(t,p,q) \Delta K(t,p,q) dq,
\end{align}
and $K(t,p,q)$ is the heat kernel, thus
\begin{align}
	\Delta_p K(t,p,q) &= \frac{\partial}{\partial t} K(t,p,q)\\
	&=-\frac{1}{4\pi}\sum_{l=0}^{\infty} l(l+1)(2l+1)e^{-l(l+1) t} P_l(\cos d).
\end{align}
Therefore, 
\begin{align}
&-\int_{S^2} K(t,p,q) \Delta K(t,p,q) dq \\
=&\frac{1}{(4\pi)^2} \sum_{l,l^\prime=0}^\infty l(l+1)(2l+1)(2l^\prime+1)e^{-\left(l(l+1)+ l^\prime(l^\prime+1)\right) t} \int_{S^2} P_l(\cos d) P_{l^\prime}(\cos d) dq.
\end{align}
To compute the last integral, we can always assume $p$ is the north pole by symmetry, so the geodesic distance between $p$ and $q$ is naturally the polar angle of $q$, that is
\begin{align}
&\int_0^{2\pi} \int_0^\pi P_l(\cos \theta) P_{l^\prime}(\cos \theta) \sin \theta d\theta d\phi \\
=&2\pi \int_{-1}^1 P_l(x) P_{l^\prime}(x) dx = \frac{4\pi}{ 2l+1} \delta_l^{l^\prime}.
\end{align}
hence
\begin{align}
\int_{S^2} \left\langle \nabla K, \nabla K\right\rangle dq &= \int_{S^2} K(t,p,q) \Delta K(t,p,q) dq \\
&= \frac{1}{4\pi} \sum_{l=0}^\infty l(l+1)(2l+1)e^{-2l(l+1)t}.
\end{align}

Therefore,
\begin{align}
	32\pi t^2 \int_{S^2} |\nabla_v K(t,p,q)|^2 dq = 4t^2\sum_{l=0}^\infty l(l+1)(2l+1)e^{-2l(l+1)t}.
\end{align}


By the isometry from $L^2(M)$ to $l^2$
\begin{align}
	U(f)_j=\int_M f(q) e_j(q) dq,
\end{align}
we have 
\begin{align}
	32\pi t^2 \sum_{i=1}^{\infty} e^{-2\lambda_i t} (\nabla e_i (v))^2  \in \left((1-\epsilon)^2,(1+\epsilon)^2\right),
\end{align}

We can truncate it at $N_0$ such that 
\begin{align}
	32\pi t^2 \sum_{i=N_0}^{\infty} e^{-2\lambda_i t} (\nabla e_i (v))^2 <\epsilon^2.
\end{align}

We consider the first $8$ eigenfunctions: $\frac{1}{2}\sqrt{\frac{3}{\pi}}\cos \theta, \frac{1}{2}\sqrt{\frac{3}{\pi}}\sin \theta \cos \phi, \frac{1}{2}\sqrt{\frac{3}{\pi}}\sin \theta \sin \phi$ $\dots$, we can obtain easily that the embedding norm, that is
\begin{align}
32\pi t^2 \sum_{i=1}^{8} e^{-2\lambda_i t} (\nabla e_i (v))^2  = 4t^2\sum_{l=1}^2 l(l+1)(2l+1)e^{-2l(l+1)t}.
\end{align}

We set $\epsilon = 0.05$ and $t=0.25$, then this embedding norm $\|df\| \in (0.95,1.05)$.

In addition, to control the error between the truncated heat kernel and the heat kernel, we need
\begin{align}
	\|K_{N_1}(t,p,q) - K(t,p,q)\| = |\sum_{i=N_1}^{\infty} e^{-\lambda_i t} e_i(p) e_i(q)| < \epsilon^\prime,
\end{align}
where $\epsilon^\prime= \frac{1}{32 \pi t}\exp\left(-\frac{\beta^2 t}{4}-\frac{2\sqrt{3dt} \beta}{3}\right) = \frac{1}{32 \pi t}$.

By spherical harmonic addition theorem, we have
\begin{align}
\sum_{i=N_1}^{\infty} e^{-\lambda_i t} e_i(p) e_i(q) = \frac{1}{4\pi}\sum_{l=l_1}^{\infty}e^{-l(l+1)t}(2l+1) P_l(\cos d).
\end{align}

We can compute $\epsilon^\prime$ and the summation above numerically, and we find that $l_1=3$ makes truncation error is smaller than $\epsilon^\prime$.

As a result, our embedding setting is $t=0.25$ and $m=8$.

\subsubsection{Verification for Inequality (\ref{star})}

We compute the second fundamental form of $\varphi(M)$ with ambient manifold $\R^m$. Through the numerical computation, the local reach for $\varphi(S^2)$ is approximately $0.646924$. 

Substituting coefficients into the \cref{star}, it is
\begin{align}
	8\tau^2_{l,\varphi(S^2)} 
	\geq \frac{9(1+0.05)^2\times 0.25}{2}\left(\log\left(8 \pi C_1(S^2)\right)\right),
\end{align}
by numerical computation, we have $C_1(S^2)$ is approximately $0.408912$. After computation, the left side is approximately $3.34809$, the right side is approximately $2.88982$. Therefore, there exist manifolds ensures that \cref{star} holds.

\end{document}